\documentclass{article}

\usepackage[preprint]{neurips_2026}

\usepackage[utf8]{inputenc} 
\usepackage[T1]{fontenc}    
\usepackage{hyperref}       
\usepackage{url}            
\usepackage{booktabs}       
\usepackage{amsfonts}       
\usepackage{nicefrac}       
\usepackage{microtype}      
\usepackage{xcolor}         
\usepackage{graphicx}
\usepackage{amsmath}
\usepackage{subcaption}
\usepackage{todonotes}
\usepackage{algorithm}
\usepackage{amsthm}
\usepackage{algorithmic}
\usepackage{bbm}
\usepackage{bm}
\usepackage{makecell}

\newtheorem{theorem}{Theorem}
\newtheorem{lemma}{Lemma}
\newtheorem{proposition}{Proposition}
\newtheorem{corollary}{Corollary}

\theoremstyle{definition}

\DeclareMathOperator*{\argmax}{argmax}
\DeclareMathOperator*{\argmin}{argmin}
\DeclareMathOperator*{\E}{\mathbb{E}}
\newcommand{\KL}{D_{\mathrm{KL}}}

\title{The Geometry of Learning to Avoid Interventions}

\author{%
  Ethan Pronovost\thanks{Correspondence: \texttt{ethan27@cs.washington.edu}} \hspace{0.2mm} \textsuperscript{1}, Khimya Khetarpal\textsuperscript{2,3}, Siddhartha Srinivasa\textsuperscript{1} \\
  \textsuperscript{1} Paul G. Allen School of Computer Science \& Engineering,  University of Washington \\
  \textsuperscript{2} Google DeepMind \\
  \textsuperscript{3} Mila \\
}

\begin{document}

\maketitle

\begin{abstract}


Human interventions are a common source of supervision in autonomous systems during deployment. Many existing approaches are based on avoiding interventions, yet the consequences of this objective are not well understood.
We develop a geometric perspective on intervention learning that characterizes intervention avoidance as constraining policies to a face of the occupancy measure polytope. This view reveals that the effectiveness of intervention learning depends on the \emph{informativeness} of the intervention strategy: highly informative interventions uniquely determine the solution, while weak interventions leave a large set of feasible policies, many of which are suboptimal.
Motivated by this under-specification, we define \emph{Robust Intervention Learning} (RIL) as the problem of learning policies that perform well under varying levels of intervention informativeness. From the geometric formulation, we derive \emph{Residual Intervention Fine-Tuning} (RIFT), which combines interventions with a prior policy to select among feasible solutions. We show that RIFT provides provable improvement over the prior and corresponds to solving a constrained reinforcement learning problem with an induced reward.
Empirically, RIFT yields consistent policy improvement across a range of intervention settings, particularly when interventions are sparse or weakly informative. These results highlight the importance of accounting for intervention informativeness and suggest a principled path toward robust learning from human feedback.

\end{abstract}

\section{Introduction}
\label{sec:introduction}

Human supervision of robotic deployments often takes the form of intervention: a human monitors the robotic system and intervenes when it behaves unacceptably \citep{hg-dagger2019}. The human's decision to intervene or not provides a direct signal on which behaviors are unacceptable \citep{eil_spencer_2020}.
However, unlike demonstrations or rewards, this intervention signal provides only partial information: it indicates that a behavior should not continue but does not specify what should be done instead. There can be many behaviors that avoid interventions and yet are sub-optimal.

\begin{figure}[t]
  \vskip 0.2in
  \centering
  \begin{minipage}[b]{0.45\textwidth}
    \centering
    \begin{subfigure}{\textwidth}
      \includegraphics[width=0.9\textwidth]{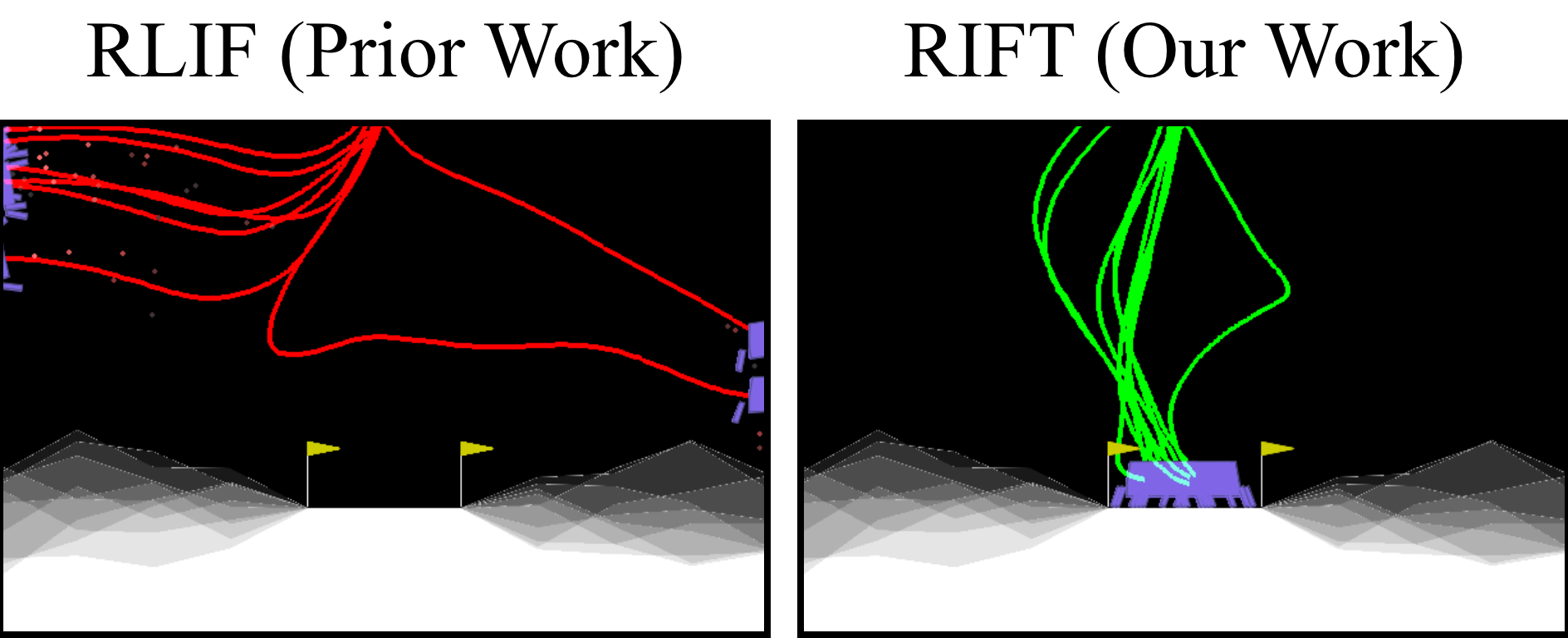}
      \caption{Example rollouts.}
      \label{fig:ll_example_rollouts}
    \end{subfigure}
    
    \begin{subfigure}{\textwidth}
      \centering
      \includegraphics[width=0.8\textwidth]{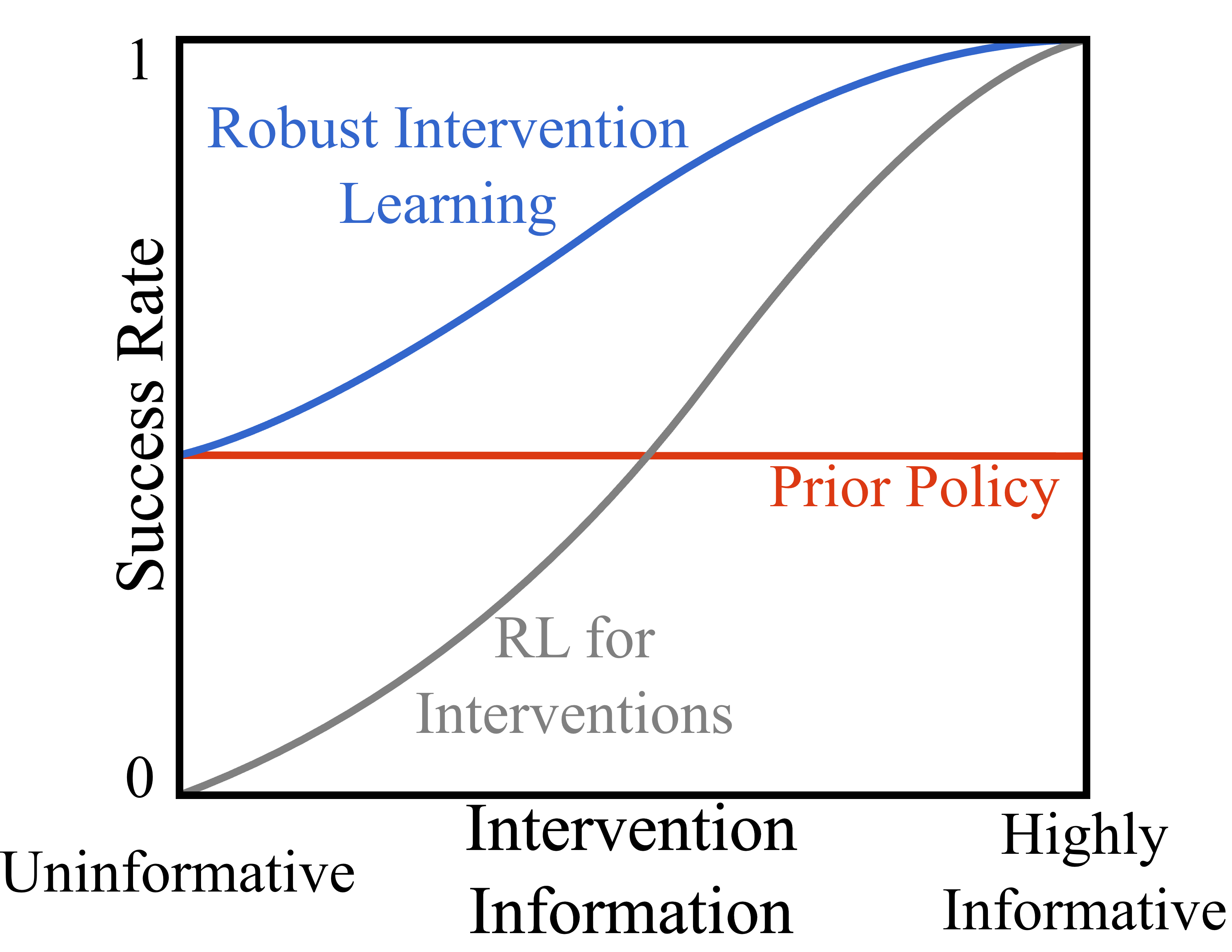}
      \caption{Robust intervention learning.}
      \label{fig:conceptual_ril}
    \end{subfigure}
  \end{minipage}
  \hfill
  \begin{minipage}[b]{0.5\textwidth}
    \centering
    \begin{subfigure}{\textwidth}
      \centering
      \includegraphics[width=0.95\textwidth]{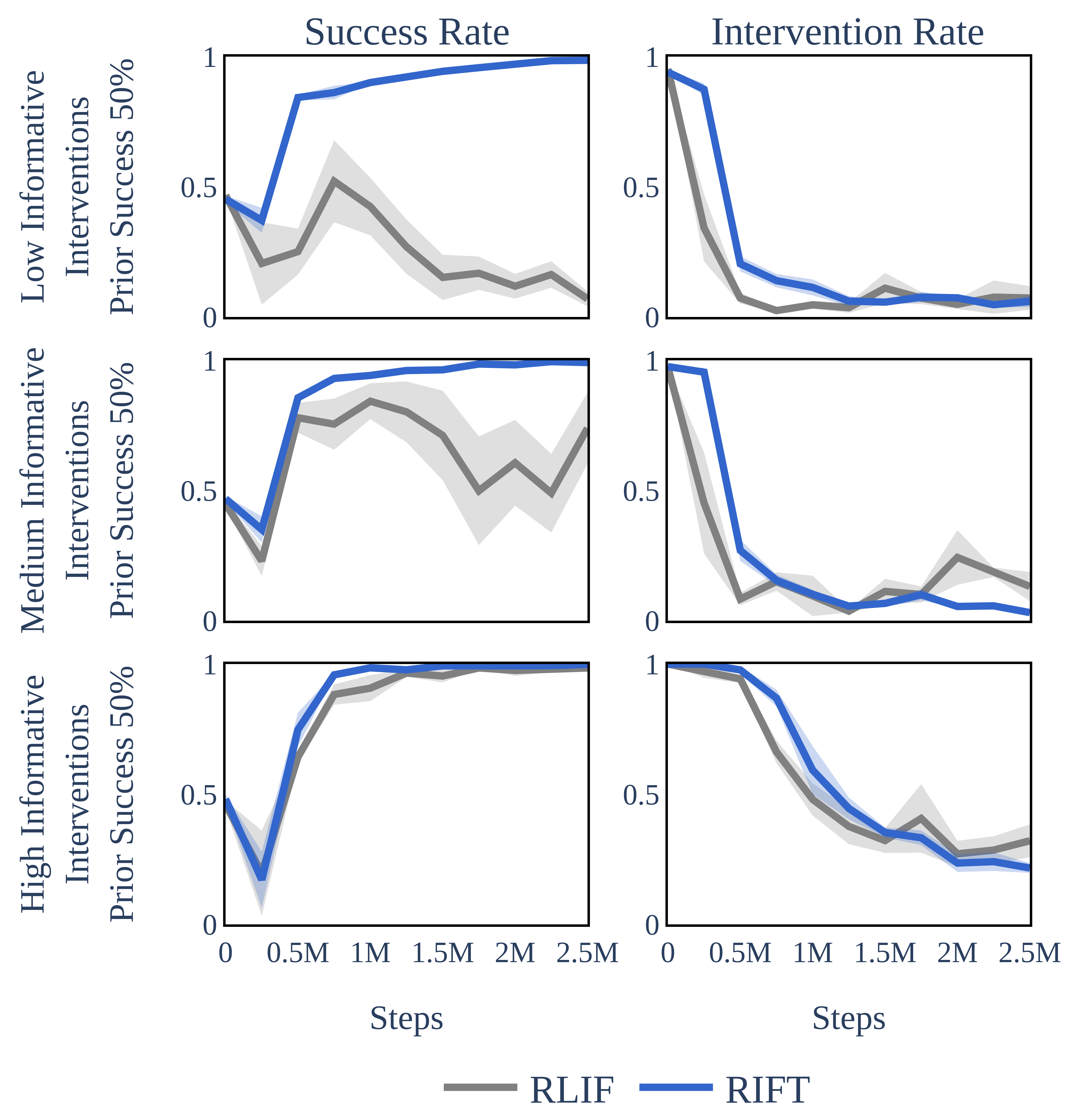}
      \caption{Performance across intervention strategies.}
      \label{fig:ll_result_simple}
    \end{subfigure}
  \end{minipage}
\caption{
\textbf{Intervention learning under varying levels of intervention informativeness.}
In Fig.~\ref{fig:ll_example_rollouts}, the intervention strategy prevents the lander from descending too quickly near the ground.
While this rules out clearly undesirable behavior, it does not fully specify the task: a policy can avoid interventions by either landing successfully or avoiding the ground altogether.
RLIF, which optimizes intervention avoidance alone, learns the latter behavior, while RIFT successfully completes the task.
Fig.~\ref{fig:conceptual_ril} illustrates the central perspective of this work: intervention strategies vary in informativeness, and robust intervention learning should improve policy performance across this spectrum. Fig.~\ref{fig:ll_result_simple} empirically reflects this structure.
RLIF performs well when interventions are highly informative, but degrades as the intervention signal becomes weaker.
In contrast, RIFT provides consistent policy improvement across a broad range of intervention strategies.}
\vspace{-9mm}
\end{figure}

A common approach is to learn policies that avoid interventions~\citep{trial_without_error_saunders_2017, eil_spencer_2020, rlif_luo_2024}.
However, a principled theoretical understanding of this approach is lacking compared to the strong theoretical foundations of DAgger-based methods~\citep{dagger_ross_2011, hg-dagger2019, korkmaz2025mile}. Notably, existing methods based on avoiding interventions implicitly rely on a strong assumption: that the intervention signal is sufficiently informative to uniquely determine the desired behavior. Under this assumption, minimizing interventions reliably leads to improving task performance.
However, this assumption is neither stated nor satisfied in many practical settings.
When interventions are less informative, merely avoiding interventions can lead to policies that satisfy the constraint without solving the task. This setting is illustrated in Fig.~\ref{fig:ll_example_rollouts}, where the intervention strategy prevents the lander from descending too quickly near the ground. While this constraint rules out clearly undesirable behavior, it does not uniquely determine the task. A policy can avoid interventions by landing successfully or avoiding the ground altogether.


In this work, we characterize intervention learning through a geometric lens.
We show that interventions define a linear constraint on the space of state-action visitation distributions, restricting policies to a face of the occupancy measure polytope.
The size and structure of this face capture the \emph{informativeness} of the intervention strategy.
This characterization formalizes when intervention avoidance uniquely determines a solution and when it leaves the problem under-specified.

This perspective suggests that intervention learning is fundamentally a problem of selecting a policy within a constrained set.
We formalize this as \emph{Robust Intervention Learning} (RIL): the problem of learning policies that improve task performance using intervention feedback without assuming that intervention avoidance uniquely specifies the solution.
Fig.~\ref{fig:conceptual_ril} illustrates this perspective.
RIL aims to achieve meaningful policy improvement across the spectrum of possible intervention strategies, rather than relying on a narrow assumption about intervention quality.

From this formulation, we derive \emph{Residual Intervention Fine-Tuning} (RIFT), which combines intervention signals with a prior policy to select among intervention-avoiding behaviors.
We show that RIFT yields provable improvement over the prior and corresponds to solving a constrained reinforcement learning problem with an induced reward.
The empirical results in Fig.~\ref{fig:ll_result_simple} reflect this structure.
Methods based solely on avoiding interventions perform well when the intervention signal is highly informative, but degrade when the signal is weaker.
In contrast, RIFT achieves consistent policy improvement across this spectrum.

\begin{figure}[ht!]
  \centering
  \begin{subfigure}[b]{0.32\textwidth}
    \centering
    \includegraphics[width=0.8\textwidth]{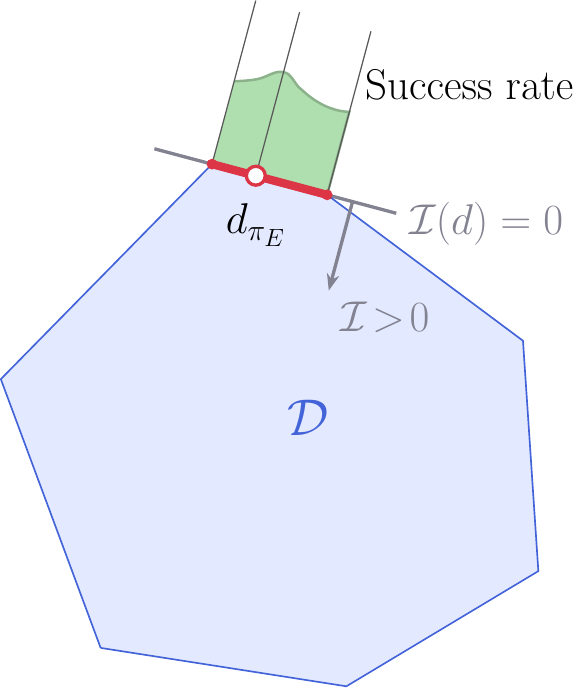}
    \caption{Good case}
    \label{fig:int-good-poly}
  \end{subfigure}
  \hfill
  \begin{subfigure}[b]{0.32\textwidth}
    \centering
    \includegraphics[width=\textwidth]{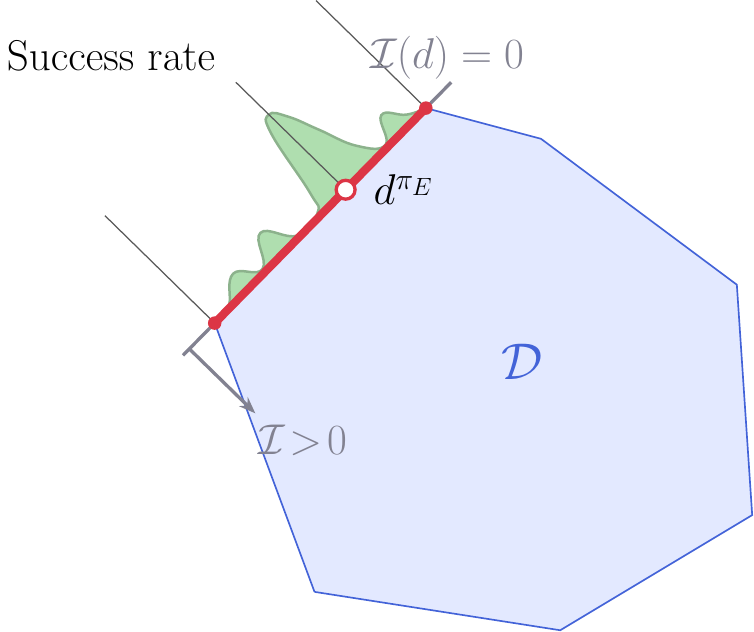}
    \caption{Bad case}
    \label{fig:int-bad-poly}
  \end{subfigure}
  \hfill
  \begin{subfigure}[b]{0.32\textwidth}
    \centering
    \includegraphics[width=\textwidth]{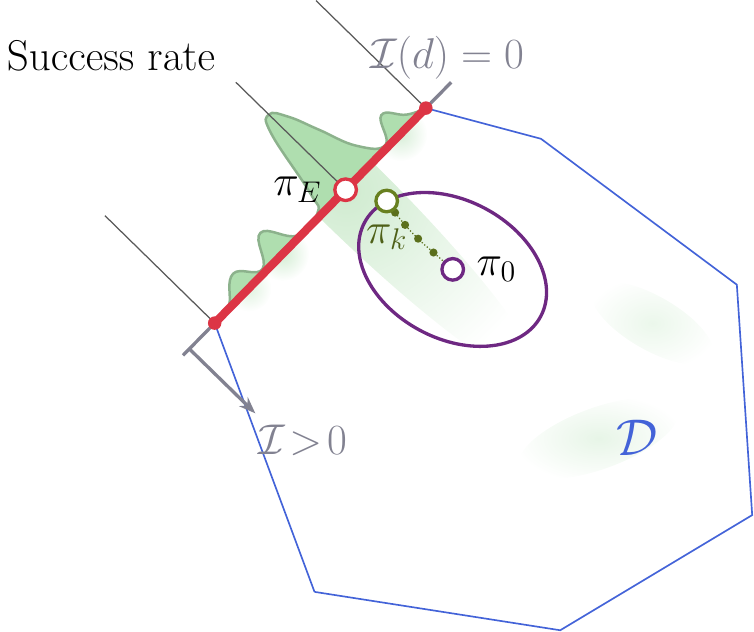}
    \caption{RIFT}
    \label{fig:rift-poly}
  \end{subfigure}
  \caption{\textbf{The geometry of intervention learning.} Avoiding interventions restricts policies to a face of the occupancy polytope.
  When in luck (Fig.~\ref{fig:int-good-poly}) this is a small face where all policies are quite successful.
  In general (Fig.~\ref{fig:int-bad-poly}) this is a large face which contains many bad policies that avoid interventions without solving the task.
  RIFT (Fig.~\ref{fig:rift-poly}) leverages complementary signals from the prior policy $\pi_0$ and interventions to push towards the expert policy.}
  \label{fig:polytopes}
  \vspace{-4mm}
\end{figure}

We evaluate our approach using \emph{emergency stop (e-stop) interventions} \citep{mo_states_ainsworth_2019}, where the supervisor provides only a binary intervention signal.
This setting is the simplest form of human supervision to deploy and isolates the core challenge of learning to avoid interventions.

Our contributions are threefold. First, we provide a geometric characterization of intervention learning that formalizes the role of intervention informativeness in Sec.~\ref{sec:geometry}. Second, we define Robust Intervention Learning and derive a simple algorithm RIFT with provable guarantees in Sec.~\ref{sec:rift}. Third, we empirically validate these predictions across a range of intervention settings in Sec.~\ref{sec:experiments}.

\section{Preliminaries}

\subsection{Markov Decision Process}
We consider a standard Markov decision process (MDP) $\mathcal M = \left( \mathcal S, \mathcal A, r, \mathcal T, \mu, \gamma \right)$ where $\mathcal S$ is the state space, $\mathcal A$ is the action space, $r : \mathcal S \times \mathcal A \to \mathbb R$ is the reward function, $\mathcal T : \mathcal S \times \mathcal A \to \Delta_{\mathcal S}$ is the transition function, $\mu \in \Delta_{\mathcal S}$ is the initial state distribution, and $\gamma \in [0, 1)$ is a temporal discount factor. We denote policies as $\pi : \mathcal S \to \Delta_{\mathcal A}$. We write $\pi(s)$ instead of $\pi(\cdot \mid s)$ when the meaning is clear. For a policy $\pi$, the infinite horizon state-action visitation distribution is defined as 
\begin{equation} \label{eqn:sa-visitation}
d^\pi (s, a) \doteq (1 - \gamma) \sum_{t=0}^\infty \gamma^t \Pr \left[ s_t=s \land a_t =a \mid \pi \right] 
\end{equation}
The infinite horizon state visitation is similarly defined, and can be expressed as $d^\pi (s) = \sum_a d^\pi (s,a)$.

\subsection{Problem Setting}
\label{sec:problem-setting}
In this work, we consider the setting where the true reward function $r^*$ is unknown. Instead, we have limited access to an \emph{interactive} expert who knows the true optimal policy $\pi^*$. Our goal is to provably and efficiently recover $\pi^*$ associated with $r^*$, through interactions with the expert under appropriate conditions on the intervention strategy. We focus on one of the simplest forms of expert interaction: an emergency stop the expert triggers when ``something bad happens''~ \citep{mo_states_ainsworth_2019}. 

We describe intervention strategies as functions $\phi : \mathcal S \times \mathcal A \to [0, 1]$ that define the probability of intervention in a given state-action pair.
We can restrict this strategy to be deterministic by limiting the output to be either 0 or 1.
While we use this model for theoretical analysis, our empirical experiments include non-Markovian intervention strategies (e.g. human reaction delay) as well.

We assume access to a prior policy $\pi_0$, consistent with \citet{rlif_luo_2024}. In practice, this prior policy is often obtained via offline training (e.g. behavior cloning). While we conduct experiments with a range of prior policies, the main setting to consider is when the prior policy has some useful task information but is still far from optimal (hence the need for intervention learning).

\subsection{Residual Q-Learning as a Fine-Tuning Primitive}
Residual Q-Learning (RQL) \citep{rql_li_2023} updates a policy $\pi_0$ optimal for the reward function $r_0$ to get a new policy $\pi_1$ optimal for reward function $r_1 = \omega r_0 + r_R$, where $\omega >0$ and $r_R$ is a residual reward. RQL introduces a modified Bellman equation for the residual soft-Q function $Q_R = Q_1 - \omega Q_0$. Defining $\widetilde Q_R (s, a) = \frac{1}{\alpha} Q_R(s, a) + \omega \log \pi_0(a \mid s)$,
we have that
\begin{equation}
\label{eqn:rql-residual-bellman}
Q_R(s, a) = r_R(s, a) + \gamma \E_{s'} \left[ \alpha \log \sum_{a' \in \mathcal A} \exp \widetilde Q_R(s', a') \right] 
\end{equation}
This residual Bellman equation only depends on $\pi_0$ and $r_R$: we do not need to know $r_0$ or $Q_0$ explicitly.
Hence, $\pi_0$ can come from imitation learning (e.g. behavior cloning) if the true reward function is unknown.
Given $Q_R$ satisfying Eqn.~\ref{eqn:rql-residual-bellman}, the new policy $\pi_1 (a \mid s) \propto \exp \widetilde Q_R(s, a)$.

\begin{figure}[t]
\begin{minipage}[t]{0.45\textwidth}
    \begin{algorithm}[H]
        \caption{Rollout with Emergency Stop}\label{alg:supervised-rollout}
        \begin{algorithmic}
        \REPEAT
          \STATE Select action $a$ from policy $\hat \pi(\cdot \mid s)$
          \STATE Advance state $s' \sim \mathcal T(\cdot \mid s,a)$
          \STATE Sample $e \sim \textrm{Bernoulli}(\phi(s, a))$
          \IF{E-stop $e$ is triggered} 
            \STATE \textbf{break} 
          \ENDIF
          \STATE Update $s \leftarrow s'$
        \UNTIL{Rollout is done.}
        \end{algorithmic}
    \end{algorithm}
\end{minipage}
\hfill
\begin{minipage}[t]{0.53\textwidth}
    \begin{algorithm}[H]
      \caption{Residual Intervention Fine-Tuning}
      \label{alg:ril}
      \begin{algorithmic}
        \REPEAT
        \STATE Initialize an empty dataset $\mathcal D$
        \REPEAT
        \STATE Rollout $\hat\pi$ with intervention using Alg \ref{alg:supervised-rollout}
        \STATE Update $\mathcal D$ with $\{(s_i, a_i, s_{i+1}, - e_i)\}_{i=0}^N$
        \UNTIL{Sufficient data collected.}
        \STATE Update policy $\hat\pi \leftarrow \text{RQL}(\hat\pi ,\mathcal D)$
        \UNTIL{Performance of $\hat\pi$ is acceptable.}
      \end{algorithmic}
    \end{algorithm}
\end{minipage}
\vspace{-0.6cm}
\end{figure}

\section{Geometry of Intervention Learning}
\label{sec:geometry}

In this section, we develop a geometric perspective of intervention learning based on state-action visitations.
In Sec.~\ref{sec:geometry-intro} we show how interventions constrain the policy to a face of the polytope of realizable state-action visitations.
The size and orientation of this face corresponds to how informative the intervention strategy is.
In Sec.~\ref{sec:geometry-opt-for-interventions}, we show how regularization acts as a selection mechanism to pick a specific point within this constraint face, and that regularization based on the prior policy $\pi_0$ offers provable intervention improvement.
Finally, in Sec.~\ref{sec:geometry-constrained-mdp}, we show how these regularized objective functions can be viewed as performing maximum-entropy RL over a constrained MDP where RLIF corresponds to optimizing the zero reward function and our proposed regularization corresponds to using $\log \pi_0$ as the reward function.
Intuitively, if $\pi_0$ captures useful task information within this constrained MDP, our proposed objective function will result in a better policy.

\subsection{Interventions as Geometric Constraints}
\label{sec:geometry-intro}
To understand the interaction between intervention avoidance and recovery imitation, we analyze the problem in occupancy space.
Let $\mathcal D \doteq \{ d^\pi : \pi \in \Pi \}$ denote the set of realizable state-action visitations.
Throughout the geometric analysis we assume that the policy class $\Pi$
contains all stationary stochastic policies.
Under this assumption the realizable occupancy set $\mathcal D$ is a convex and bounded polytope~\citep{altman2021constrained}.
When $\Pi$ is restricted (e.g., deterministic policies or function-approximated policies),
the realizable occupancy set need not be convex.
In that case the geometric arguments below should be interpreted
as describing the idealized occupancy geometry induced by the full stochastic policy class.
Define the intervention functional for an occupancy $d \in \mathcal D$ as
\begin{equation}
\label{eqn:intervention-functional}
\mathcal I(d) \doteq \E_{(s,a)\sim d}\left[\phi(s,a)\right].
\end{equation}

\begin{proposition}[Intervention Defines a Face]
If $\Pi$ contains all stochastic policies, then $\mathcal D$ is convex.
Assume $\phi \neq 0$. Since $\mathcal I(d)$ is linear in $d$, the minimizer set
\begin{equation}
\mathcal D_{\min} \doteq \argmin_{d\in\mathcal D} \mathcal I(d) 
\end{equation} 
is the intersection of $\mathcal D$
with a supporting hyperplane.
\end{proposition}
Henceforth we'll call $D_{\min}$ as the \emph{minimal-intervention face}. This formalizes intervention avoidance as a geometric constraint in occupancy space.
The strength of this constraint depends on the intervention strategy. If $\pi^*$ is a deterministic policy and an intervention occurs for any action that differs from $\pi^*$ (i.e. $\phi(s, a) = \mathbbm{1} \left[ a \neq \pi^*(s) \right]$), then $\mathcal D_\textrm{min} = \{ d^{\pi^*} \}$. On the other hand, if the intervention strategy only intervenes to avoid a specific bad state-action pair (i.e. $\phi(s, a) = \mathbbm{1} \left[ s = s_\textrm{bad} \land a = a_\textrm{bad} \right]$), then $\mathcal D_\textrm{min}$ contains a large number of policies.

The \emph{informativeness} of the intervention strategy is how tightly it constrains the policies to solve the task.
If $r^\star$ is the (unknown) true reward function, we can measure this as $\sup_{d, d' \in \mathcal D_\textrm{min}} \langle d - d', r^\star \rangle$.
When the minimal-intervention face is narrow in reward-relevant directions, simply being on the face is enough to ensure good task performance. When the face remains large, however, it is possible to be on the minimal-intervention face and still have poor task performance.

\subsection{Optimization for Intervention Learning}
\label{sec:geometry-opt-for-interventions}

This geometric perspective highlights the limitations of solely avoiding interventions: $\argmin_{d} \mathcal I(d)$ is an \emph{under-specified problem} for all but the most highly informative intervention strategies. 
We need some additional selection mechanism to identify an occupancy from $\mathcal D_\textrm{min}$.

A natural choice is to employ a continuously differentiable and strictly convex regularizer $\Psi : \mathcal D \to \mathbb R$ with regularization strength $\omega$ to define the objective function $\mathcal I(d) + \omega \Psi(d)$.
Since $\mathcal I(d)$ is linear and $\Psi(d)$ is strictly convex, the overall objective is strictly convex as well and admits a unique minimizer over $\mathcal D$. The strict convexity of the regularizer transforms intervention minimization from a constraint problem with a flat solution manifold into a well-posed optimization problem with a unique solution.
Define $d^\star_\Psi(\omega) \doteq \argmin_{d \in \mathcal D} \mathcal I(d) + \omega \Psi(d)$.
As shown in the following lemma, this mapping traces a continuous path through $\mathcal D$ that approaches $\mathcal D_\textrm{min}$.

\begin{lemma}[Path of Regularized Solutions]
\label{lem:path-of-reg-solutions}
Suppose $\Psi$ is strictly convex and continuously differentiable on $\mathcal D$. Then the mapping $d^\star_\Psi$ is continuous on $(0, \infty)$ and has the limit
\begin{equation}
\lim_{\omega \to 0^+} d^\star_\Psi (\omega) = \argmin_{d \in \mathcal D_\textrm{min}} \Psi(d)
\end{equation}
\end{lemma}
The proof is given in Sec.~\ref{apdx:proof-path-of-reg-solutions}.
The next lemma shows that if this regularization is centered on the prior policy $\pi_0$, optimizing $\mathcal I(d) + \omega \Psi(d)$ achieves provable intervention improvement.

\begin{lemma}[Intervention Improvement]
\label{lem:intervention-improvement}
Assume $\Psi$ is a strictly convex regularizer with its minimum at $d^{\pi_0}$.
Let $\omega>0$, and define $d^\textrm{opt} \doteq \argmin_{d \in \mathcal D} \mathcal I(d) + \omega \Psi(d)$.
Then $\mathcal I(d^\textrm{opt}) \le \mathcal I(d^{\pi_0})$, with strict inequality unless $d^{\pi_0} \in \mathcal D_\textrm{min}$ already.
\end{lemma}

The proof is given in Sec.~\ref{apdx:proof-intervention-improvement}.
RLIF \citep{rlif_luo_2024} uses maximum-entropy RL \citep{deep_energy_based_policies_haarnoja_2017}, and is equivalent to optimizing $\mathcal I(d) + \omega \mathcal H(d)$ where $\mathcal H(d) \doteq \E_{(s, a) \sim d} \left[  \log \frac{d(s,a)}{d(s)} \right] = \E_{s \sim d} \left[ - \mathcal H[\pi_d(s) ] \right]$ is an entropy regularizer on the policy $\pi_d$ corresponding to $d$.

\begin{corollary}
Optimizing $\mathcal I(d) + \omega \mathcal H(d)$ can cause the intervention rate to get worse if $\omega$ is too large.
\end{corollary}

These findings suggest that regularization $\Psi$ should involve $\pi_0$ in order to achieve reliable policy improvement. In the next section we offer another perspective on why using $\pi_0$ in the objective function is beneficial when the intervention strategy provides useful but incomplete task information.

\subsection{Intervention-Constrained MDP}
\label{sec:geometry-constrained-mdp}

The preceding sections show that intervention learning reduces to selecting a visitation distribution from the minimal-intervention face $\mathcal D_\textrm{min}$.
To give another perspective on this selection mechanism, we first characterize the structure of $\mathcal D_{\min}$. 
Define sets for the states and actions that can appear in minimal-intervention visitation distributions:
\begin{equation}
\mathcal G(s) \doteq \left \{ a \in \mathcal A \mid \exists \, d \in \mathcal D_\textrm{min} \textrm{ such that } d(s, a) > 0 \right \} \qquad \mathcal S' \doteq \left \{ s \in \mathcal S \mid \mathcal G(s) \neq \emptyset \right \}
\end{equation}
$\mathcal G(s)$ is not merely the set of actions where $\phi(s,a) = 0$. An action may satisfy $\phi(s,a)=0$ and still fail to belong to $\mathcal G(s)$ if it leads to states from which an intervention is unavoidable.
The following lemma shows that $\mathcal D_\textrm{min}$ exactly corresponds to a constrained MDP with states $\mathcal S'$ and actions $\mathcal G(s)$.

\begin{lemma}[Intervention-Constrained MDP]
\label{lem:structure-dmin}
Assume an intervention-free policy exists. Then
\begin{equation}
\mathcal D_{\min}
=
\left\{
d^\pi \;\middle|\;
\pi(a\mid s)=0
\;\; \forall s\in\mathcal S',\; a\notin\mathcal G(s)
\right\}.
\end{equation}
\end{lemma}

The proof is given in Sec.~\ref{apdx:proof-structure-dmin}.
We next characterize what regularization does on this constrained MDP. For any reference policy $\pi_q$, we can define a strictly convex regularizer over $\mathcal D$:
\begin{equation}
\Theta(d \parallel d^{\pi_q})
\doteq
\E_{s\sim d}
\left[
\KL\!\left(\pi_d(\cdot\mid s)\,\|\,\pi_q(\cdot\mid s)\right)
\right].
\end{equation}

\begin{theorem}[Objectives on the Constrained MDP]
\label{thm:induced-objective}
Assume an intervention-free policy exists. Then
\begin{equation}
\lim_{\omega \to 0^+} \argmin_{d \in \mathcal D} \mathcal I(d) + \omega \Theta(d \parallel d^{\pi_q} )
\end{equation}
is equivalent to solving maximum-entropy reinforcement learning on the intervention-constrained MDP with reward $r_q(s,a)=\log \pi_q(a\mid s)$.
\end{theorem}

The proof is given in Sec.~\ref{apdx:proof-induced-objective}.
The regularizer determines the objective after intervention constraints are enforced. Entropy regularization corresponds to $\Theta (d \parallel d^{\pi_\textrm{unif}})$ and therefore induces a zero reward function. Reverse KL centered at a prior policy $\pi_0$ induces reward $\log \pi_0(a\mid s)$.

\begin{corollary}[RLIF on the Constrained MDP]
Assume there exists an intervention-free policy. Then the limit as $\omega \to 0^+$ of $\argmin_{d \in \mathcal D} \mathcal I(d) + \omega \mathcal H(d)$ is equivalent to performing maximum-entropy reinforcement learning on the intervention-constrained MDP with zero reward.
\end{corollary}

The interventions define the constrained MDP, and regularization induces an objective on this reduced domain. The quality of the selected policy depends on whether that induced objective is aligned with the task structure that remains after interventions have ruled out inadmissible behavior.
We argue that for most practical intervention settings, $\log \pi_0$ is a better induced objective than the zero reward function.
In Sec.~\ref{sec:experiments}, we show that this holds empirically across a wide range of settings.
The role of regularization as an induced reward function is further characterized in Sec.~\ref{apdx:additional-geometry-proofs}.

In this work we consider the original and most general version of intervention learning in which the true reward is entirely unknown.
However, from the results in this section it is clear how additional task information can be incorporated into the objective function as additional reward terms.

\section{Residual Intervention Fine-Tuning}
\label{sec:rift}

From the previous section, using $\mathcal I(d) + \omega \Theta (d \parallel d^{\pi_0})$ as the objective function offers both provable policy improvement and a better reward function in the intervention-constrained MDP. In this section we propose a practical algorithm for optimizing the aforementioned objective.

\subsection{Residual Fine-Tuning for Interventions}
\label{sec:residual-fine-tuning}

The reverse-KL objective over state-action visitations is equivalent to maximizing
\begin{equation}
\label{eqn:j-int-objective}
J_\textrm{INT}(\pi) = \E_{s_t,a_t} \left[ \sum_{t=0}^\infty \gamma^t \Bigl( -\phi(s_t, a_t) - \omega \KL \left[ \pi(s_t) \parallel \pi_0 (s_t) \right] \Bigr) \right]
\end{equation}
This is equivalent to performing RL fine-tuning \citep{rafailov2023dpo} on $\pi_0$ with reward $-\phi$.
By casting intervention learning as a fine-tuning problem, we can leverage existing methods for fine-tuning.
The following lemma connects Equation~\ref{eqn:j-int-objective} with Residual Q-Learning \citep{rql_li_2023}.

\begin{lemma}
\label{lem:rql-equivalence}
Let $\pi_0$ be a policy with full support over actions.
Set the entropy coefficient $\alpha = \omega$.
Then 
\begin{equation}
\textrm{RQL}(\pi_0, -\phi) = \argmax_{\pi} J_\textrm{INT}(\pi)
\end{equation}
\end{lemma}
The proof is in Sec.~\ref{apdx:proof-rql-equivalence}.
We can view $J_\textrm{INT}$ as adding a residual reward term $-\phi$ to the implicit reward function of $\pi_0$ to get the fine-tuned policy $\hat \pi = \textrm{RQL}(\pi_0, - \phi)$.
Using an off-policy algorithm allows us to re-use previously collected intervention data.
Human intervention data is typically expensive to collect, so being able to reuse old data is a major practical benefit.

Our ensuing Algorithm~\ref{alg:ril} is deceptively simple. A dataset of transitions is obtained by performing rollouts with interventions according to Algorithm~\ref{alg:supervised-rollout}.
These transitions are labeled with a reward that is -1 if an intervention occurred and 0 otherwise.
The expected reward in a given state-action pair is thus $-\phi(s, a)$. The policy $\hat \pi$ is updated via RQL using these transitions. 
Different stopping criteria for the two loops of Algorithm~\ref{alg:ril} could be used.
For example, the inner loop could be executed for a fixed number of rollouts, and the outer loop could be executed until the rate of interventions falls below a threshold.
If the expert executes actions after intervening, these transitions can be added to the dataset as further examples of transitions without interventions.

The specific residual Q-learning algorithm used (e.g. residual soft Q-learning \citep{deep_energy_based_policies_haarnoja_2017} or residual soft actor-critic \citep{sac_haarnoja_2018}) will depend on the given task (e.g. whether the action space is discrete or continuous).
This algorithm can be run both online and offline.

\subsection{Directions of Policy Improvement}
\label{sec:policy-improvement}

Viewing this algorithm as applying a residual reward term $-\phi$ to $\pi_0$ allows us to characterize intervention strategies that bring $\pi_0$ closer to the expert policy $\pi_E$.
In Sec.~\ref{apdx:directions-of-policy-improvement}, we show how intervention strategies based on the advantage difference, state-action visitation difference, and state visitation difference can all bring the learner policy closer to the expert policy.
While we do not have an explicit equation for the $\phi$ implemented by human experts, these results motivate how ``intervening when the policy does something a human expert would not do'' leads to policy improvement.

\section{Experiments}
\label{sec:experiments}

Since this work focuses on understanding intervention avoidance, our experiments use emergency stop interventions across a range of settings.
We simulate the human expert using an optimal policy trained via standard RL, similar to \citet{eil_spencer_2020, rlif_luo_2024}.
We use a variety of environments, prior policies, and intervention strategies.
These experiments are designed to test how well the geometric perspective described in Sec.~\ref{sec:geometry} translates to practice.
We also provide practical tips for stable residual fine-tuning from interventions.

\begin{figure}[t]
  \centering
  \begin{subfigure}[b]{0.32\textwidth}
    \centering
    \includegraphics[width=\textwidth]{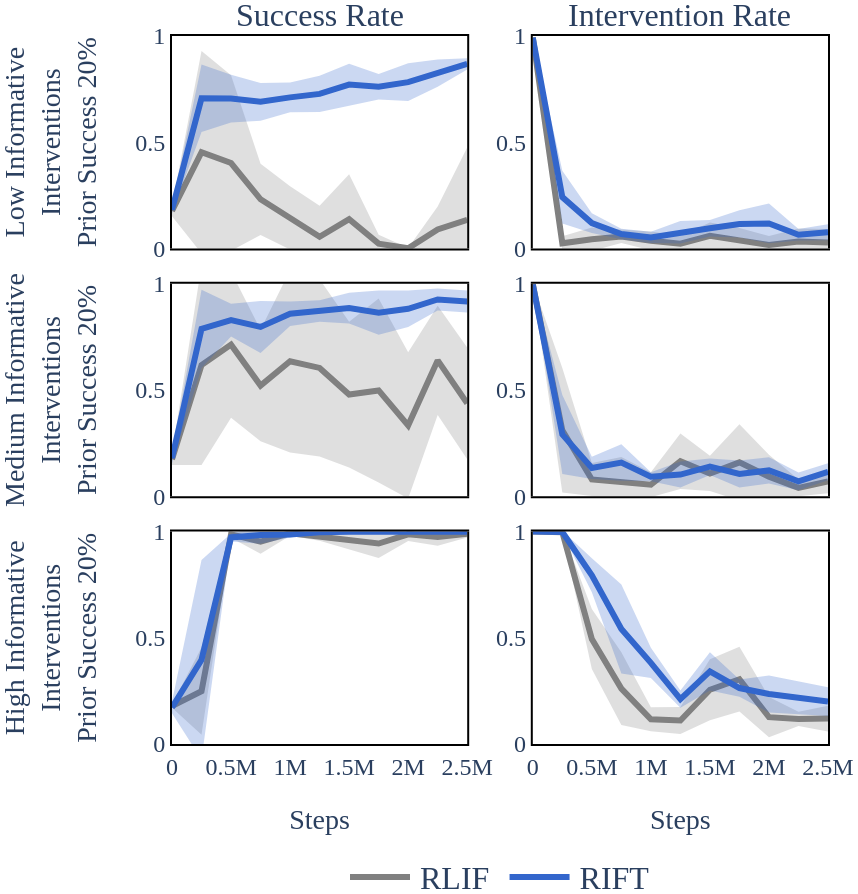}
  \end{subfigure}
  \begin{subfigure}[b]{0.32\textwidth}
    \centering
    \includegraphics[width=\textwidth]{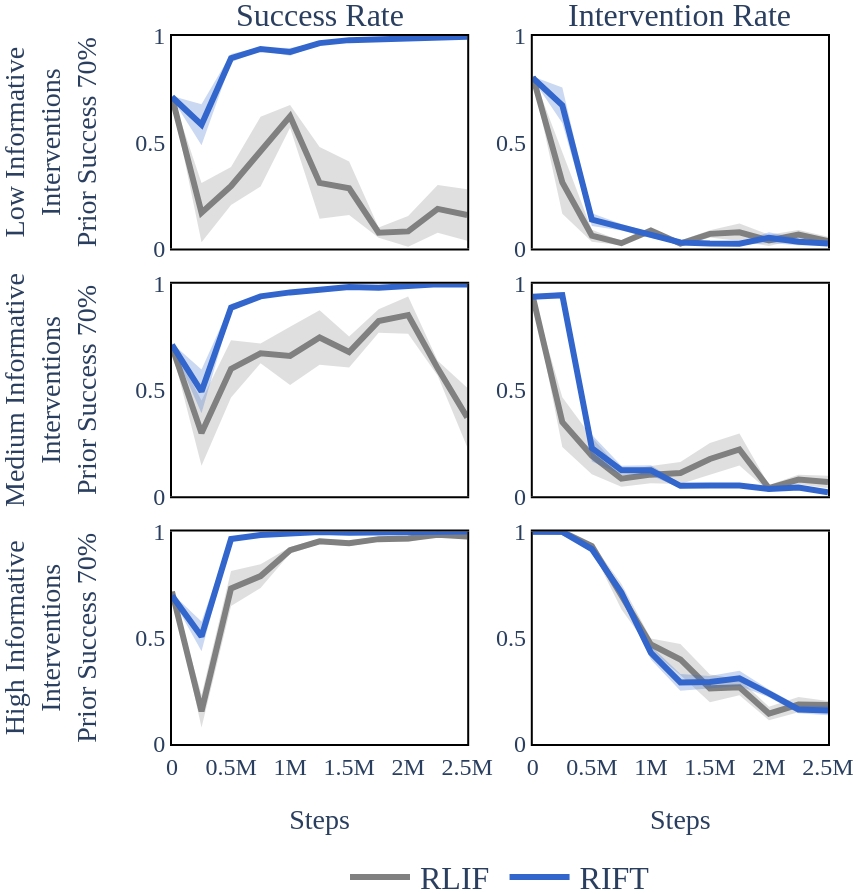}
  \end{subfigure}
  \begin{subfigure}[b]{0.32\textwidth}
    \centering
    \includegraphics[width=\textwidth]{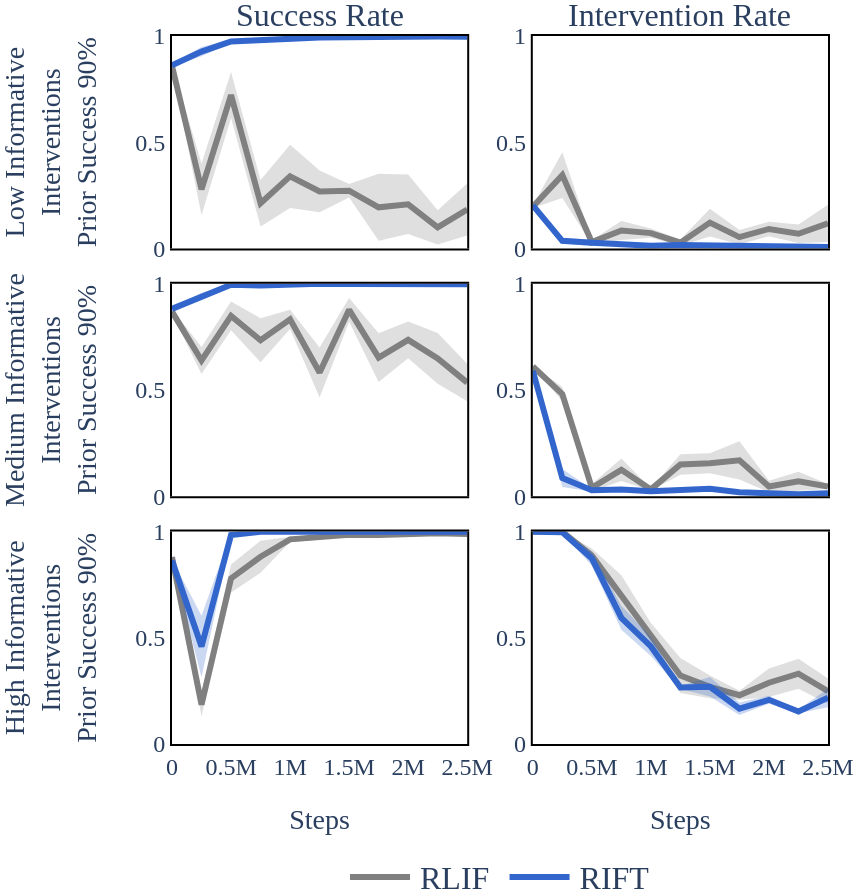}
  \end{subfigure}
  \caption{\textbf{RIFT achieves consistent policy improvement, while RLIF requires highly informative interventions.} Results in Lunar Lander using different prior policies and intervention strategies.}
  \label{fig:lunar-lander-big-grid}
  \vspace{-4mm}
\end{figure}


\subsection{Experimental Setup}
\label{sec:experimental-setup}
We perform experiments using Gymnasium \citep{gym_towers_2024}.
We treat RLIF~\citep{rlif_luo_2024} as the unregularized baseline corresponding to $\omega = 0$.
The default intervention model is based on the difference in Q-values, similar to \citet{rlif_luo_2024}:
\begin{equation}
\label{eqn:q-diff-intervention-strategy}
\phi(s, a) = \mathbf{1} \left[ Q^*(s, \pi^*(s)) - Q^*(s, a) > B \right]
\end{equation}
Here the threshold $B >0$ represents the suboptimality gap allowed before an intervention occurs.
As $B \to 0$ this intervention strategy becomes more informative, and as $B \to \infty$ the interventions become less informative.
Additional experimental details are given in Sec.~\ref{apdx:experimental-setup}.
We also experiment with different intervention models, as described below.

\begin{figure}[t]
  \centering

  \begin{subfigure}{0.335\textwidth}
    \centering
    \includegraphics[width=\textwidth]{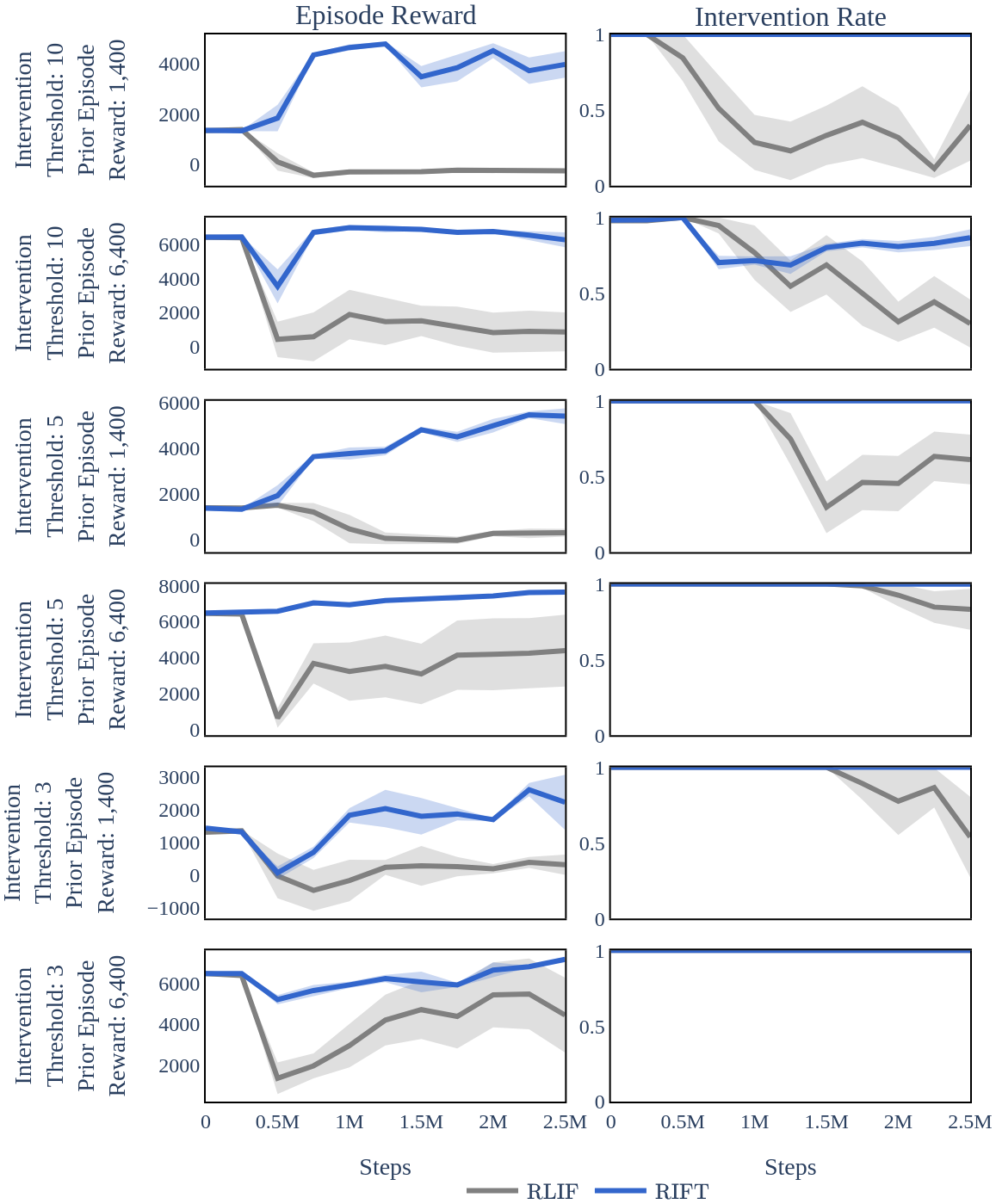}
    \caption{Half cheetah.}
    \label{fig:half-cheetah}
  \end{subfigure}
  \hfill
  \begin{subfigure}{0.32\textwidth}
     \centering 
     \includegraphics[width=\textwidth]{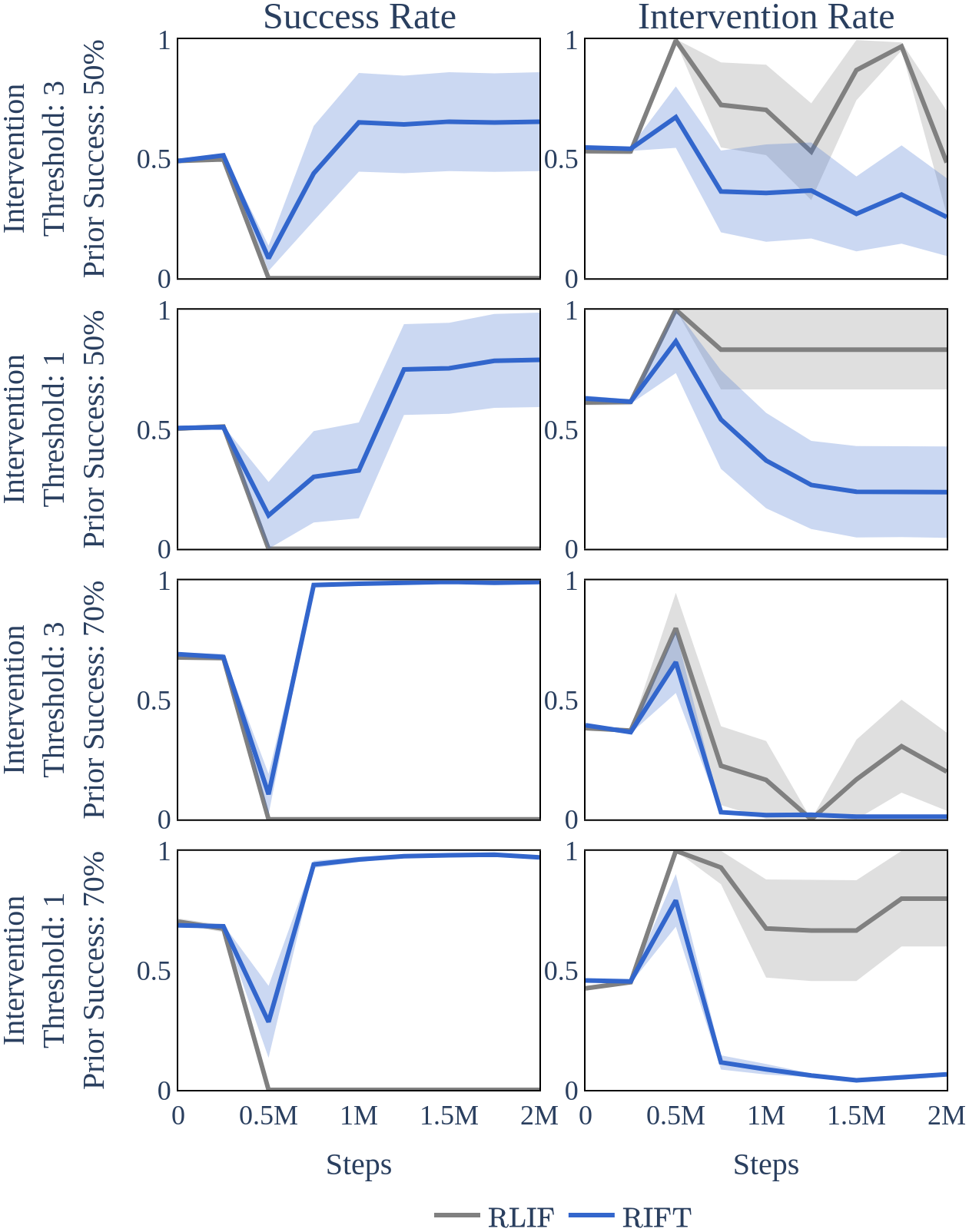}
     \caption{Bipedal walker.}
     \label{fig:bipedal-walker}
  \end{subfigure}
  \hfill
  \begin{subfigure}{0.30\textwidth}
     \centering 
     \includegraphics[width=\textwidth]{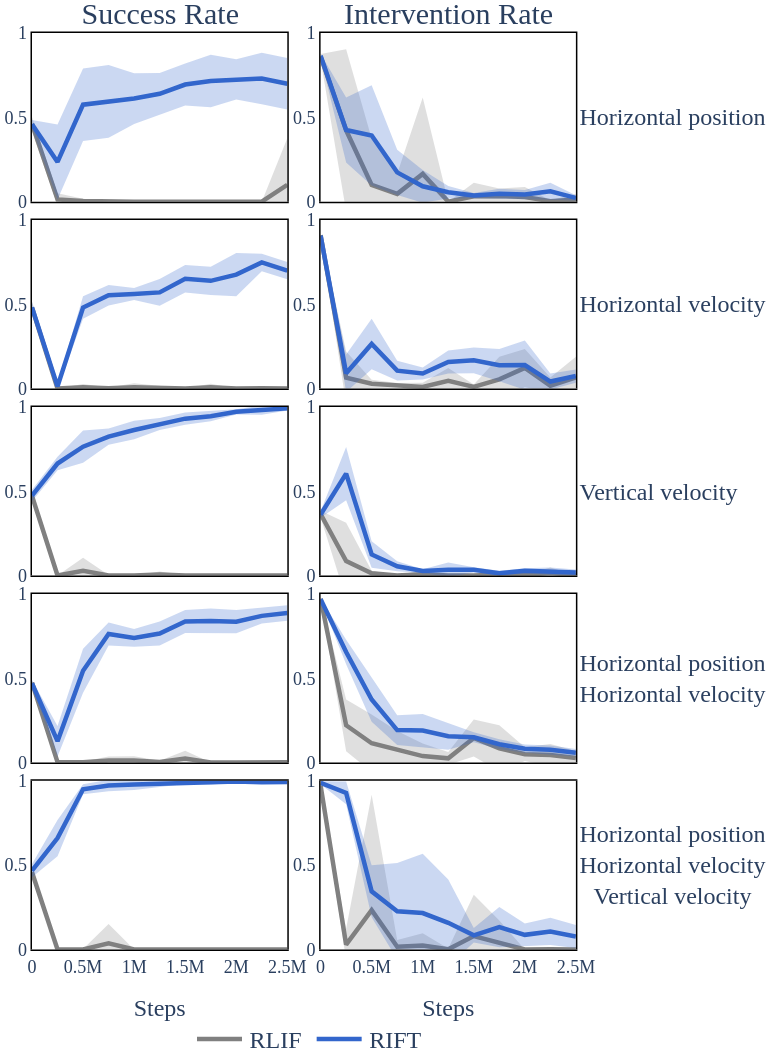}
     \caption{Heuristic interventions.}
     \label{fig:heuristic-interventions-main}
  \end{subfigure}
\caption{\textbf{RIFT consistently outperforms RLIF across environments, intervention strategies, and prior policies.}
Fig.~\ref{fig:half-cheetah}: Half Cheetah with Q-value-based interventions using different thresholds. Fig.~\ref{fig:bipedal-walker}: Bipedal Walker with Q-value-based interventions using different thresholds. Fig.~\ref{fig:heuristic-interventions-main}: Lunar Lander with physics-based heuristic interventions.
}
\vspace{-6mm}
\end{figure}

\subsection{Performance Across Intervention Settings}
We compare RLIF and RIFT across environments, prior policies, and intervention strategies (via the $B$ threshold).
Quantitative results are shown in Fig.~\ref{fig:ll_result_simple}, Fig.~\ref{fig:lunar-lander-big-grid}, Fig.~\ref{fig:half-cheetah}, and Fig.~\ref{fig:bipedal-walker}, with qualitative examples in Fig.~\ref{fig:half-cheetah-video}.
RIFT consistently outperforms RLIF across these domains.
The gap is most pronounced under low-informative interventions and with a high quality prior, just as theory predicts.



\subsection{Breaking the Intervention Model}
\label{sec:breaking-intervention-model}

We test how RIFT performs under intervention settings that violate the assumptions used in Sec.~\ref{sec:geometry}. We first experiment with non-Markovian intervention delay, where the intervention happens a fixed number of steps after the intervention criteria are met.
If the episode ends during the delay period, the intervention happens on the final timestep.
By adding delay, the intervention strategy becomes less informative due to challenges in credit assignment.
As seen in Table~\ref{tab:delay}, RIFT is robust to intervention delay while RLIF performance deteriorates.

In Sec.~\ref{apdx:noisy-interventions} we experiment with stochastic interventions, breaking the assumption that an intervention-free policy exists.
RIFT is robust to high levels of both false positive and false negative noise, consistently outperforming RLIF.
In Sec.~\ref{apdx:heuristic-interventions} we construct intervention strategies that overlap with the expert policy visitation distribution. RIFT converges to new ways of solving the task not observed in either the expert or prior policies.

\subsection{Interpretable Intervention Strategies}

The intervention model described in Eqn.~\ref{eqn:q-diff-intervention-strategy} acts globally: all states have the potential for intervention.
This property does not universally hold for all intervention strategies.
For example, a human can be quick to intervene in safety-critical states, but be much more relaxed otherwise.
The locality of the intervention strategy is another aspect that impacts its informativeness.

We define physics-based heuristic intervention strategies that operate only in safety-critical states.
Fig.~\ref{fig:ll_example_rollouts} shows policies trained using an intervention strategy based on the rate of descent and proximity to the ground.
While both RLIF and RIFT learn to avoid interventions, RLIF does so by avoiding the ground while RIFT learns to land successfully.
This difference is clear in the quantitative results shown in Fig.~\ref{fig:heuristic-interventions-main} as well.
Further details and results are described in Sec.~\ref{apdx:heuristic-interventions}.
Fascinatingly, if we modify the intervention heuristics to deliberately overlap with the expert state distribution, RIFT is able to learn new behaviors to solve the task not observed in either the expert or prior policies.

\subsection{Regularization Strength}

Lemma~\ref{lem:path-of-reg-solutions} suggests that the optimal solution to the RIFT objective should be fairly stable in a window of $\omega$ values close to 0.
In practice, however, having $\omega$ be too small could require significantly more data to accurately estimate the difference of very small reward terms.
To test this, we perform an ablation over $\omega$, as shown in 
Fig.~\ref{fig:omega_ablation}.
We find that near-optimal performance can be achieved with $\omega$ values in a 1 or 2 order of magnitude window.
Furthermore, a much larger window of $\omega$ values still yields results comparable or superior to those of RLIF.
This suggests that $\omega$ does not need be extensively tuned, as getting within a factor of 10 seems to be sufficient.
For a given prior policy, the optimal choice of $\omega$ decreases as the interventions become more informative, reflecting the diminished importance of regularization when the minimal-intervention face is highly constrained.

\begin{figure}[t]
  \centering

  \begin{subfigure}{0.4\textwidth}
    \centering
    \includegraphics[width=\textwidth]{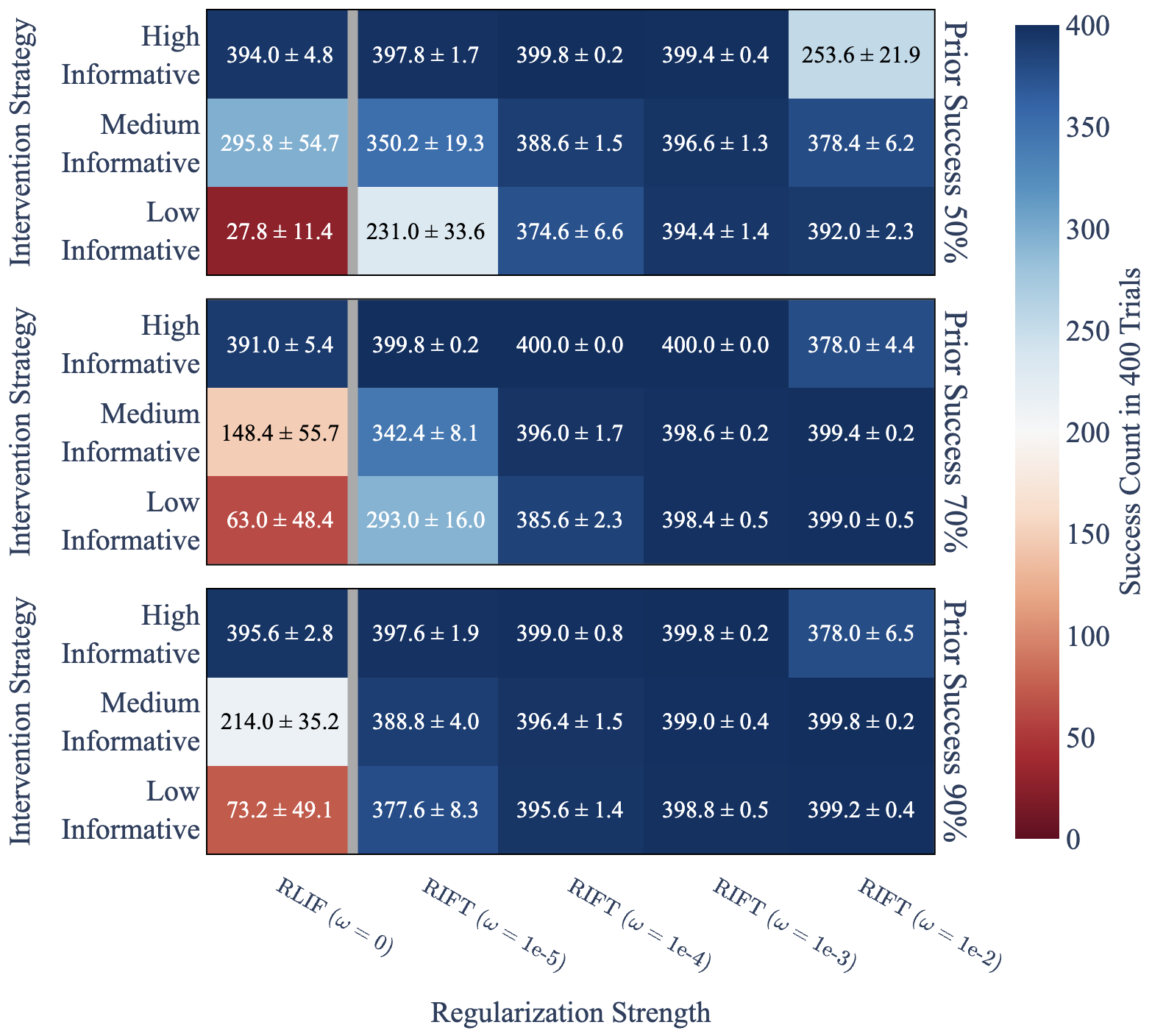}
    \caption{Ablating regularization strength}
    \label{fig:omega_ablation}
  \end{subfigure}
  \hfill
  \begin{subfigure}{0.55\textwidth}
     \centering 
     \includegraphics[width=\textwidth]{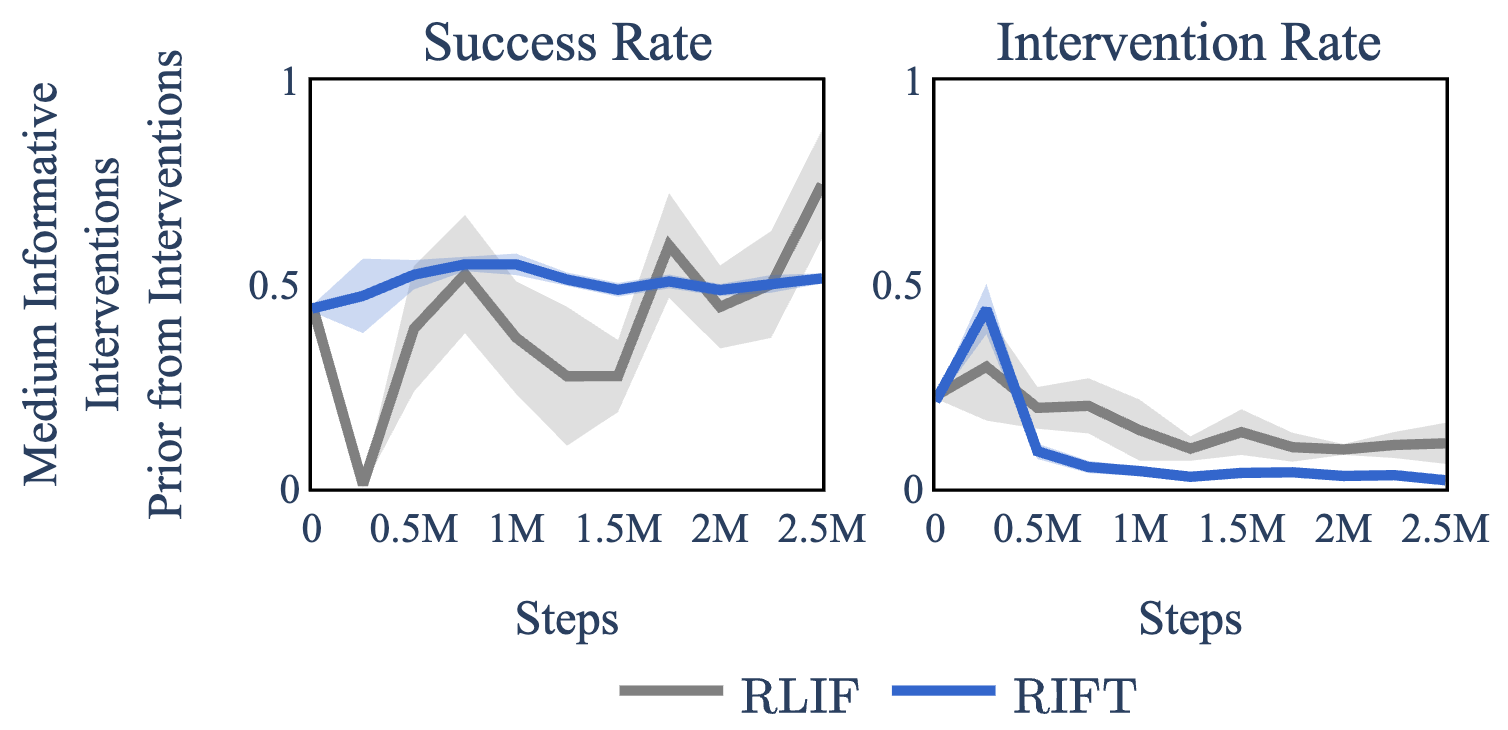}
     \caption{Prior policy constructed to be redundant with interventions.}
     \label{fig:failure-case-intervention-prior}
  \end{subfigure}
\caption{
\textbf{Interaction between regularization strength and intervention informativeness.} In Fig.~\ref{fig:omega_ablation} we ablate the regularization coefficient $\omega$.
Near-optimal performance can be achieved with $\omega$ values over several orders of magnitude, suggesting this parameter does not need to be extensively tuned.
In Fig.~\ref{fig:failure-case-intervention-prior} we construct a non-uniform prior policy that does not provide a meaningful induced reward on the constrained MDP, resulting in near-constant success rate from RIFT.
}
\vspace{-6mm}
\end{figure}

\subsection{Constructed Counterexample}

If the uniform policy is used as the prior then the RIFT objective reduces to the RLIF objective.
Using the theory of Sec.~\ref{sec:geometry}, we can construct other prior policies that we predict should make RIFT perform similarly to RLIF.
We train a policy from scratch using standard maximum-entropy RL on the intervention reward.
From Theorem~\ref{thm:induced-objective}, this prior policy should not induce a useful reward function on the intervention-constrained MDP, as it contains no information other than the intervention signal.
Indeed, we see in Fig.~\ref{fig:failure-case-intervention-prior} that although this prior policy achieves nearly a 50\% success rate, RIFT does not meaningfully change the policy's performance.

\section{Related Work}
In intervention learning, the expert interacts with the policy during training. 
DAgger~\citep{dagger_ross_2011} is a core algorithm in this space, in which the expert provides action labels in states visited by the policy.
HG-DAgger~\citep{hg-dagger2019} uses an interaction format in which the takes control of the robot during the intervention.
This algorithm only leverages the expert actions, ignoring the information implicit in when the expert decided to intervene.
This can result in learning to correct mistakes instead of learning to avoid them in the first place~\citep{eil_spencer_2020}.

An alternative approach is to learn to avoid interventions using the binary intervention signal~\citep{trial_without_error_saunders_2017,rlif_luo_2024}.
Other works incorporate both the binary signal and expert actions, but do so with custom optimization functions instead of using standard MDP objectives \citep{eil_spencer_2020, korkmaz2025mile}.
By considering e-stop interventions, our work focuses on learning from the binary intervention signal.
This is most similar to RLIF \citep{rlif_luo_2024}, with the key difference that this work does not assume a specific, highly informative intervention strategy is used.

\section{Discussion and Limitations}
\label{sec:discussion}
This work provides a novel theoretical understanding of learning to avoid interventions.
We show that interventions alone do not fully specify a task, and that combining intervention avoidance with the prior policy in the objective function yields stronger theoretical justification and empirical performance.
We instantiate this approach in \emph{Residual Intervention Fine-Tuning} (RIFT).
More broadly, the geometric perspective proposed in this work suggests a unifying lens for deployment-time learning, where feedback, priors, and objectives all partially describe the task.
We believe this view opens the door to richer forms of intervention, adaptive supervision, and scalable human-in-the-loop learning.

One limitation of our work is that we do not conduct experiments with human subjects as the expert.
We find simulation experiments better suited for testing the predictions of our theory, as it is difficult to precisely control the intervention setting with human subjects.
We leave applying RIFT to real human feedback as an exciting direction for future work.

By explicitly modeling the limits of intervention feedback, our approach aims to reduce unsafe or brittle behavior during deployment.
We do not introduce new forms of data collection or human monitoring beyond standard intervention protocols already used in robotic systems.
We expect the primary societal impact of this work to be improved safety and reliability in autonomous systems that operate under human supervision.



\bibliographystyle{abbrvnat}
\bibliography{neurips_2026}


\appendix

\newpage
\section{Background}

\subsection{Maximum Entropy Objective}
\label{apdx:max-ent}
The maximum-entropy objective \citep{deep_energy_based_policies_haarnoja_2017, ziebart_thesis_2010} for a given temperature $\alpha \in (0, \infty)$ is 
\begin{equation} \label{eqn:max-ent-objective}
J_\textrm{MaxEnt}(\pi \mid r) = \E_{s_t, a_t} \Biggl[ \sum_{t=0}^\infty \gamma^t \Bigl( r(s_t, a_t) + \alpha \mathcal H [ \pi(s_t)] \Bigr) \Biggr]
\end{equation}
where variables $\{(s_t, a_t) \}_{t=0}^\infty$ are sampled as $s_0 \sim d$, $a_t \sim \pi( s_t)$, and $s_{t+1} \sim \mathcal T(\cdot \mid s_t, a_t)$. Here $\mathcal H[\pi(s)]$ is the entropy of the policy at state $s$.
The soft-Q function $Q^\pi$ captures the expected reward from a given state-action pair under policy $\pi$:
\begin{equation} \label{eqn:soft-q-definition}
Q^\pi (s_0, a_0) = r(s_0, a_0) + \E_{s_t, a_t} \left[ \sum_{t=1}^\infty \gamma^t \Bigl( r(s_t, a_t) + \alpha \mathcal H \left[ \pi( s_t) \right] \Bigr) \right]
\end{equation}
where $s_t$ and $a_t$ for $t \ge 1$ are sampled as in (\ref{eqn:max-ent-objective}).
Similarly, we define the soft-value function 
\begin{equation} \label{eqn:v-star}
V^\pi(s) = \E_{a \sim \pi(s)} \Bigl[ Q^\pi (s, a) + \alpha \mathcal H \left[ \pi (s) \right] \Bigr]
\end{equation}
The soft-advantage is the difference between soft-Q function and soft-value function: $A^\pi(s, a) = Q^\pi(s, a) - V^\pi(s)$.

\subsection{Single Timestep Visitation}

Define single-timestep state-action visitations as 
\begin{equation}
d^\pi_t(s, a) \doteq \Pr \left[ s_t = s \land a_t = a \mid \pi \right]
\end{equation}
and similarly $d^\pi_t (s) = \sum_{a \in \mathcal A} d_t^\pi (s, a)$.

\section{Directions of Policy Improvement}
\label{apdx:directions-of-policy-improvement}

Sec.~\ref{sec:residual-fine-tuning} shows that the intervention objective $J_\textrm{INT}$ is equivalent to applying Residual Q-Learning \citep{rql_li_2023} to $\pi_0$ with residual reward $-\phi$.
This connection allows us to consider choices for $\phi$ that will bring the policy closer to the expert policy $\pi_E$.
We provide two different characterizations of such provable intervention strategies. In the first, interventions occur when the robot prefers an action that the expert does not.
In the second, interventions occur when the robot reaches a state that the expert does not. These characterizations align with how we would intuitively expect an expert to behave.

\subsection{Advantage Difference}
Using a residual reward of $\alpha \log \pi_E(a \mid s) - \alpha \log \pi_0(a \mid s)$ (plus reward shaping terms) will yield $\textrm{RQL}(\pi_0, \alpha \log \pi_E - \alpha \log \pi_0) = \pi_E$.
This represents the difference in advantage predicted by the expert versus the prior policy.
Intervention strategies are restricted to $\phi(s, a) \in [0, 1]$, but if $-\phi$ is close to such a reward, then the result of fine-tuning should approach $\pi_E$.
While there are adversarial MDPs where this is not true, it holds fairly well for real-world MDPs.
This amounts to having $\phi(s, a)$ be high whenever $\pi_0(a \mid s) > \pi_E(a \mid s)$.
If we assume that the student is proposing actions where $\pi_0(a \mid s)$ is high, this would correlate with intervening whenever $\pi_E(a \mid s)$ is low.

\subsection{Visitation Difference}
Framing the policy update as a residual reward allows us to use the gradients of $\hat \pi$ with respect to the residual reward to get a first-order estimate of the impact of our policy update.
Consider the following imitation objective:
\begin{equation} \label{eqn:weighted-kl}
\Lambda (\pi) \doteq \sum_{s \in \mathcal S} d^{\pi_E}(s) D_\textrm{KL} \left[ \pi_E( s) \parallel \pi( s) \right]
\end{equation}
Clearly $\pi_E = \argmin_{\pi} \Lambda(\pi)$, so minimizing $\Lambda$ will lead us to $\pi_E$.
This objective relates to $J_\textrm{MaxEnt}(\pi_E \mid \hat r)$, where $\hat r$ is the implicit reward corresponding to $\pi$.

\begin{lemma} \label{thm:irl-to-weighted-kl}
Let $\hat r$ be a reward function with corresponding optimal policy $\hat \pi = \argmax_\pi J_\textrm{MaxEnt} (\pi \mid \hat r)$ and soft-value function $\hat V$.
Then
\begin{equation}
J_\textrm{MaxEnt}(\pi_E \mid \hat r) = \E_{s_0 \sim \mu} \left[ \hat V(s_0) \right] - \frac{\alpha}{1 - \gamma} \Lambda(\hat \pi)
\end{equation}  
\end{lemma}

\begin{proof}
Rewriting (\ref{eqn:max-ent-objective}) in terms of state visitations,
\begin{align}
&J_\textrm{MaxEnt}(\pi_E \mid \hat r) \\
&= \E_{s_t, a_t} \sum_{t=0}^\infty \gamma^t \left( \hat r(s_t, a_t) + \alpha \mathcal H \left[ \pi_E( s_t) \right] \right) \\
&= \frac{1}{1 - \gamma} \sum_{s \in \mathcal S} d^{\pi_E} (s) \sum_{a \in \mathcal A} \pi_E(a \mid s) \left( \hat r(s, a) + \alpha \mathcal H \left[ \pi_E( s) \right] \right)  \nonumber
\end{align}
Using the relationship between reward function, policy, and value function \citep{bridging_gap_value_policy_rl_nachum_2017}, we can expand the inner sum as
\begin{align}
&\sum_{a \in \mathcal A} \pi_E(a \mid s) \left( \hat r(s, a) + \alpha \mathcal H \left[ \pi_E( s) \right] \right) \\
&=\sum_{a \in \mathcal A} \left( \pi_E(a \mid s) \cdot \hat r(s, a) \right) + \alpha \mathcal H \left[ \pi_E( s) \right]  \\
&= \sum_{a \in \mathcal A} \pi_E(a \mid s) \left( \alpha \log \hat \pi (a \mid s) + \hat V(s) - \gamma \E_{s'} \left[ \hat V(s') \right] - \alpha \log \pi_E(a \mid s) \right) \nonumber \\
&= - \alpha D_\textrm{KL} \left[ \pi_E( s) \parallel \hat \pi( s) \right] + \hat V(s) - \gamma \E_{s'} \left[ \hat V(s') \right] \nonumber
\end{align}
Thus, the overall objective is
\begin{equation}
J_\textrm{MaxEnt}(\pi_E \mid \hat r) = \frac{1}{1 - \gamma} \sum_{s \in \mathcal S} d^{\pi_E}(s) \left( \hat V(s) - \gamma \E_{s'} \left[ \hat V(s') \right] - \alpha D_\textrm{KL} \left[ \pi_E( s) \parallel \hat \pi( s) \right] \right)
\end{equation}

The soft-value terms can be simplified as
\begin{align}
&\frac{1}{1 - \gamma}\sum_{s \in \mathcal S} d^{\pi_E}(s) \left( \hat V(s) - \gamma \E_{s'} \left[ \hat V(s')\right] \right) \\
&= \sum_{s \in \mathcal S} \sum_{t=0}^\infty \gamma^t d^{\pi_E}_t (s) \hat V(s) - \sum_{\substack{s \in \mathcal S \\ a \in \mathcal A \\ s' \in \mathcal S}} \sum_{t=0}^\infty \gamma^t d^{\pi_E}_t (s) \cdot \gamma  \pi_E(a \mid s) \mathcal T(s' \mid s, a) \hat V(s') \nonumber \\
&= \sum_{s \in \mathcal S} \sum_{t=0}^\infty \gamma^t d^{\pi_E}_t (s) \hat V(s) - \sum_{t=0}^\infty \gamma^{t+1} d^{\pi_E}_{t+1} (s) \hat V(s) \\
&= \sum_{s \in \mathcal S} d_0^{\pi_E} (s) \hat V(s) \\
&= \sum_{s \in \mathcal S} \mu(s) \hat V(s) \\
&= \E_{s_0 \sim \mu}\left[ \hat V(s_0) \right]
\end{align}
Thus,
\begin{equation}
J_\textrm{MaxEnt}(\pi_E \mid \hat r) = \E_{s_0 \sim \mu} \left[ \hat V(s_0) \right] - \frac{\alpha}{1 - \gamma} \sum_{s \in \mathcal S} d^{\pi_E}(s) D_\textrm{KL} \left[ \pi_E( s) \parallel \hat \pi( s) \right] 
\end{equation}

\end{proof}

We can use the gradient of $\Lambda$ with respect to $\hat r$ to get a first-order approximation of how adding a residual reward impacts $\Lambda(\hat \pi)$.
Fortunately, this gradient simplifies to an interpretable expression.
\begin{theorem} \label{thm:grad-kl-div}
Define $\Lambda$ as in (\ref{eqn:weighted-kl}).
Let $\hat r$ be a reward function whose optimal policy is $\hat \pi$.
Then the partial derivative evaluated at the point $\hat r$ is
\begin{equation}
\frac{\partial \, \Lambda (\hat \pi)}{\partial  \, \hat r(s, a)} = \frac{1}{\alpha} \left( d^{\hat \pi}(s, a) - d^{\pi_E}(s, a) \right)
\end{equation}
\end{theorem}

Viewng these functions as $|\mathcal S| \cdot |\mathcal A|$ dimensional vectors, this result can be written as \begin{equation}
\nabla_{\mathbf {\hat r}} \Lambda (\hat \pi) = \frac{1}{\alpha} \left( \mathbf d^{\hat \pi} - \mathbf d^{\pi_E} \right) 
\end{equation}
The proof is given in Sec.~\ref{apdx:proof-grad-kl-div}.

If $\phi$ satisfies $\langle \phi, \mathbf d^{\hat \pi} - \mathbf d^{\pi_E} \rangle > 0$ then updating $\hat \pi$ with residual reward $-\phi$ can be viewed as approximating a step of gradient descent on $\Lambda(\hat \pi)$.
In words, this dot product means that interventions occur in state-action pairs that the student is likely to visit and the expert is not.
Although $\mu^{\pi_E}$ is not observable, this characterization shows that intervention strategies aligned with visitation mismatch implicitly approximate descent directions on this objective.

\subsection{State-Based Strategies}
State-based interventions form an important and practically relevant subclass of intervention strategies, for example if the expert does not directly observe the actions.
We can view a state-based intervention strategy as constraining $\phi(s, a) = \varphi(s)$ for some $\varphi : \mathcal S \to [0, 1]$.
For any policy $\pi$, let $\bm \rho^\pi$ be the $|\mathcal S|$ dimensional vector of state visitations (i.e. $\bm \rho^\pi [i] = d^\pi(s_i)$).

\begin{lemma}
Define $\Lambda$ as in (\ref{eqn:weighted-kl}) and constraint $\phi$ to only depend on the state: $\phi(s, a) = \varphi(s)$ for some function $\varphi : \mathcal S \to [0,1]$ with vector representation $\bm \varphi$.
Then
\begin{equation}
\left \langle \phi, \nabla_{\hat{\mathbf r}} \Lambda(\hat \pi)\right \rangle = \frac{1}{\alpha} \left \langle \bm \varphi,  \bm \rho^{\hat \pi} - \bm \rho^{\pi_E} \right\rangle
\end{equation}
\end{lemma}

\begin{proof}
Using the result of Theorem~\ref{thm:grad-kl-div},
\begin{align}
&\left \langle \phi, \nabla_{\mathbf{\hat r}} \Lambda(\hat \pi)\right \rangle \\
&= \frac{1}{\alpha }\sum_{\substack{s \in \mathcal S \\ a \in \mathcal A}} \phi(s, a)  \left( d^{\hat \pi} (s, a) - d^{\pi_E} (s, a) \right) \\
&= \frac{1}{\alpha }\sum_{\substack{s \in \mathcal S \\ a \in \mathcal A}} \varphi(s)  \left( d^{\hat \pi} (s, a) - d^{\pi_E} (s, a) \right) \\
&= \frac{1}{\alpha }\sum_{\substack{s \in \mathcal S}} \varphi(s) \left( d^{\hat \pi} (s) - d^{\pi^*} (s) \right) \\
&= \frac{1}{\alpha} \left \langle \bm \varphi, \bm \rho^{\hat \pi} - \bm \rho^{\pi_E} \right \rangle
\end{align}
\end{proof}

For state-based intervention strategies, the dot product between the residual reward and the gradient of $\Lambda(\hat \pi)$ is proportional to the dot product of $\varphi$ and the difference in state visitation between $\hat \pi$ and $\pi_E$.
For a state-based intervention strategy, intervening in states that the student visits frequently and the expert does not can also be viewed as a projected gradient step on $\Lambda(\hat \pi)$.
This explains why coarse interventions that only signal undesirable states can still lead to policy improvement.

\section{Proofs}
\label{apdx:proofs}

\subsection{Lemma~\ref{lem:path-of-reg-solutions}}
\label{apdx:proof-path-of-reg-solutions}
\begin{proof}
For each $\omega>0$, the objective $\mathcal I(d)+\omega \Psi(d)$ is strictly convex and continuously differentiable on the compact convex set $\mathcal D$, and admits a unique minimizer.
As the objective varies continuously with respect to $\omega$, so too the minimizer depends continuously on $\omega$ \citep{berge1963topologicalspaces}.

For any $\Psi$ and $\omega > 0$, let $d^\textrm{opt} = d^\star_\Psi(\omega)$.
By optimality, for any $d \in \mathcal D$
\begin{equation}
\mathcal I(d^\textrm{opt}) + \omega \Psi(d^\textrm{opt}) \le \mathcal I(d) + \omega \Psi(d)
\end{equation}
In particular, considering any $d^\textrm{min} \in \mathcal D_{\min}$,
\begin{equation}
\mathcal I(d^\textrm{opt}) \le \mathcal I(d^\textrm{min}) + \omega \big(\Psi(d^\textrm{min}) - \Psi(d^\textrm{opt})\big).
\end{equation}

Since $\mathcal I(d^\textrm{min}) \le \mathcal I(d^\textrm{opt})$ as well and $\Psi$ is bounded on $\mathcal D$,
taking $\omega \to 0$ forces $\mathcal I(d^\textrm{opt}) \to \mathcal I(d^\textrm{min})$.
Thus any limit point must lie in $\mathcal D_{\min}$.

Now restrict the objective $\mathcal I(d)+\omega\Psi(d)$ to $d\in\mathcal D_{\min}$.
For all $d\in\mathcal D_{\min}$ the first term $\mathcal I(d)$ is a constant, hence
\begin{equation}
\argmin_{d\in\mathcal D_{\min}} \mathcal I(d) + \omega \Psi(d) = \argmin_{d\in\mathcal D_{\min}} \Psi(d)
\end{equation}
Therefore any limit point of $d^\textrm{opt}$ as $\omega\to 0^+$ lies in
$\argmin_{d\in\mathcal D_{\min}} \Psi(d)$.
As $\Psi$ is strictly convex, this minimizer is unique.
\end{proof}

\subsection{Lemma~\ref{lem:intervention-improvement}}
\label{apdx:proof-intervention-improvement}

\begin{proof}
By optimality of $d^{\mathrm{opt}}$ and feasibility of $d^{\pi_0}$,
\begin{equation}
\mathcal I(d^{\mathrm{opt}})+\omega\Psi(d^{\mathrm{opt}})
\le
\mathcal I(d^{\pi_0})+\omega\Psi(d^{\pi_0}).
\end{equation}
Since $\Psi$ attains its unique minimum at $d^{\pi_0}$,
\begin{equation}
\Psi(d^{\mathrm{opt}})
\ge
\Psi(d^{\pi_0}).
\end{equation}
Rearranging gives
\begin{equation}
\mathcal I(d^{\mathrm{opt}})
\le
\mathcal I(d^{\pi_0}).
\end{equation}

It remains to show strictness. Since $\mathcal I$ is linear and $\mathcal D$ is convex,
$d^{\pi_0}\notin\mathcal D_{\min}$ implies that there exists $d' \in \mathcal D$ with
\begin{equation}
\mathcal I(d') < \mathcal I(d^{\pi_0}).
\end{equation}
For any sufficiently small $\eta\in(0,1)$, define
\begin{equation}
d_\eta = (1-\eta)d^{\pi_0}+\eta d' \in \mathcal D.
\end{equation}
By linearity,
\begin{equation}
\mathcal I(d_\eta)
=
(1-\eta)\mathcal I(d^{\pi_0})+\eta \mathcal I(d')
<
\mathcal I(d^{\pi_0}).
\end{equation}
Since $\Psi$ is continuous and minimized at $d^{\pi_0}$, we have
\begin{equation}
\Psi(d_\eta) \to \Psi(d^{\pi_0})
\quad \text{as } \eta\to 0.
\end{equation}
Therefore, for sufficiently small $\eta>0$,
\begin{equation}
\mathcal I(d_\eta)+\omega\Psi(d_\eta)
<
\mathcal I(d^{\pi_0})+\omega\Psi(d^{\pi_0}).
\end{equation}
Thus $d^{\pi_0}$ is not the minimizer of the regularized objective, so
$d^{\mathrm{opt}}\neq d^{\pi_0}$.
Since $\Psi$ has a unique minimizer at $d^{\pi_0}$,
\begin{equation}
\Psi(d^{\mathrm{opt}})>\Psi(d^{\pi_0}).
\end{equation}
Returning to the optimality inequality, this strict increase in $\Psi$ forces
\begin{equation}
\mathcal I(d^{\mathrm{opt}})
<
\mathcal I(d^{\pi_0}).
\end{equation}
\end{proof}

\subsection{Lemma~\ref{lem:structure-dmin}}
\label{apdx:proof-structure-dmin}

\begin{proof}
Since an intervention-free policy exists and $\phi(s,a)\ge 0$, we have
\begin{equation}
d\in\mathcal D_{\min}
\quad\Longleftrightarrow\quad
\mathcal I(d)=0.
\end{equation}

First consider any $d\in\mathcal D_{\min}$. By definition of $\mathcal G(s)$, if $a\notin\mathcal G(s)$, then no distribution in $\mathcal D_{\min}$ places positive mass on $(s,a)$. Hence $d(s,a)=0$. For any state with $d(s)>0$, the policy induced by $d$ satisfies
\begin{equation}
\pi_d(a\mid s)=\frac{d(s,a)}{d(s)}=0
\qquad
\text{for all } a\notin\mathcal G(s).
\end{equation}
For states in $\mathcal S'$ with $d(s)=0$, the policy can be defined arbitrarily on $\mathcal G(s)$ and with zero probability outside $\mathcal G(s)$.
Thus, we can define a policy $\pi$ that satisfies $\pi(a \mid s) = 0$ for all $s \in \mathcal S'$ and $a \notin \mathcal G(s)$ such that $d = d^\pi$.

Conversely, suppose $\pi(a\mid s)=0$ for all $s\in\mathcal S'$ and $a\notin\mathcal G(s)$. We will show that $d^\pi\in\mathcal D_{\min}$.

For states $s\in\mathcal S'$, the policy assigns no probability to actions outside $\mathcal G(s)$ by assumption. We now show that states outside $\mathcal S'$ have zero visitation. Since an intervention-free policy exists, every state in the support of the initial distribution must belong to $\mathcal S'$; otherwise no intervention-free visitation distribution could exist.

Assume for contradiction that $d^\pi(s')>0$ for some $s'\notin\mathcal S'$. Then there must exist a transition $(s,a,s')$ with $s\in\mathcal S'$, $\pi(a\mid s)>0$, and $\mathcal T(s'\mid s,a)>0$. But because $s'\notin\mathcal S'$, no visitation distribution in $\mathcal D_{\min}$ can place mass on $s'$. Therefore no distribution in $\mathcal D_{\min}$ can place positive mass on $(s,a)$ either, since doing so would induce positive visitation of $s'$. Hence $a\notin\mathcal G(s)$, contradicting $\pi(a\mid s)>0$.

Thus $d^\pi$ assigns zero mass to every infeasible state-action pair. Since all state-action pairs with positive intervention cost are infeasible, $\mathcal I(d^\pi)=0$, and therefore $d^\pi\in\mathcal D_{\min}$.
\end{proof}

\subsection{Theorem~\ref{thm:induced-objective}}
\label{apdx:proof-induced-objective}

We first show the following proposition.
\begin{proposition}[Convexity of Conditional Reverse KL]
\label{prop:reverse-kl-convexity}
Assume $\pi_q(a\mid s)>0$ on all feasible state-action pairs.
Then $\Theta(d\parallel d^{\pi_q})$ is strictly convex in $d$.
\end{proposition}

\begin{proof}
Observe that 
\begin{align}
\Theta(d \parallel d^{\pi_q})
&\doteq
\E_{s\sim d}
\left[
\KL\!\left(\pi_d(\cdot\mid s)\,\|\,\pi_q(\cdot\mid s)\right)
\right] \\
&= \sum_{s,a} d(s,a) \log \frac{d(s,a)}{d(s) \pi_q(a \mid s)} \\
&=
\sum_{s,a} d(s,a)\log d(s,a)
-
\sum_s d(s)\log d(s)
-
\sum_{s,a} d(s,a)\log \pi_q(a|s).
\end{align}

The first term is strictly convex in $d$, the second is convex because $d(s)=\sum_a d(s,a)$ is linear and $x\log x$ is convex on $\mathbb R_{\ge 0}$, and the third is linear.
Therefore $\Psi(d)$ is strictly convex.
\end{proof}

Having shown that $\Theta$ is strictly convex, we can now prove Theorem~\ref{thm:induced-objective}.
\begin{proof}
From Lemma~\ref{lem:path-of-reg-solutions}, this limit is $\argmin_{d \in \mathcal D_\textrm{min}} \Theta (d \parallel d^{\pi_q})$.
For $d\in\mathcal D_{\min}$, Lemma~\ref{lem:structure-dmin} implies that the corresponding policy is supported on $\mathcal S'$ and $\mathcal G(s)$. Expanding the definition of $\Theta$ gives
\begin{equation}
\Theta(d\parallel d^{\pi_q})
=
\E_{(s,a)\sim d}
\left[
\log \pi_d(a\mid s)-\log \pi_q(a\mid s)
\right].
\end{equation}
Since
\begin{equation}
\E_{(s,a)\sim d}[\log \pi_d(a\mid s)]
=
-\E_{s\sim d}\left[\mathcal H(\pi_d(\cdot\mid s))\right],
\end{equation}
we obtain
\begin{equation}
\Theta(d\parallel d^{\pi_q})
=
-\E_{s\sim d}\left[\mathcal H(\pi_d(\cdot\mid s))\right]
-
\E_{(s,a)\sim d}[\log \pi_q(a\mid s)].
\end{equation}
Therefore minimizing $\Theta$ over $\mathcal D_{\min}$ is equivalent to maximizing
\begin{equation}
\E_{(s,a)\sim d}[\log \pi_q(a\mid s)]
+
\E_{s\sim d}\left[\mathcal H(\pi_d(\cdot\mid s))\right],
\end{equation}
which is the maximum-entropy RL objective on the constrained MDP with reward $r_q(s,a)=\log \pi_q(a\mid s)$.
Note that $\pi_q$ is defined in the original MDP, so is not necessarily a normalized policy in the constrained MDP.
\end{proof}

\subsection{Evaluating the Induced Objectives}
\label{apdx:additional-geometry-proofs}
We now quantify alignment between the induced reward and the true task reward on the constrained MDP. Let $r^\star$ denote the true reward function.

\begin{lemma}[Alignment Bound]
\label{lem:alignment-bound}
Let $\widehat \pi_q$ maximize the maximum-entropy objective with reward $r_q(s,a)$ on the constrained MDP, and let $\pi^\star$ maximize the unregularized task return under $r^\star$ on the same constrained MDP. If
\begin{equation}
\sup_{\substack{s \in \mathcal S' \\ a \in \mathcal G(s)}}
\left |r_q(s,a)-r^\star(s,a) \right|
\le \epsilon,
\end{equation}
then
\begin{equation}
\E_{d^{\pi^\star}}[r^\star]
-
\E_{d^{\widehat \pi_q}}[r^\star]
\le
2 \epsilon + \log |\mathcal A|.
\end{equation}
\end{lemma}

\begin{proof}
Since $\widehat \pi_q$ maximizes the maximum-entropy objective under $r_q$,
\begin{equation}
\E_{d^{\widehat \pi_q}}[r_q] - \mathcal H(d^{\widehat \pi_q})
\ge
\E_{d^{\pi^\star}}[r_q] - \mathcal H(d^{\pi^\star}).
\end{equation}
Rearranging,
\begin{equation}
\E_{d^{\pi^\star}}[r_q]
-
\E_{d^{\widehat \pi_q}}[r_q]
\le
\mathcal H(d^{\pi^\star}) - \mathcal H(d^{\widehat \pi_q})
\le
\log |\mathcal A|.
\end{equation}
The reward approximation assumption implies, for any policy $\pi$,
\begin{equation}
\left|
\E_{d^\pi}[r^\star-r_q]
\right|
\le
\epsilon.
\end{equation}
Therefore,
\begin{align*}
\E_{d^{\pi^\star}}[r^\star]
-
\E_{d^{\widehat \pi_q}}[r^\star]
&=
\E_{d^{\pi^\star}}[r^\star-r_q]
+
\E_{d^{\pi^\star}}[r_q]
-
\E_{d^{\widehat \pi_q}}[r_q]
+
\E_{d^{\widehat \pi_q}}[r_q-r^\star] \\
&\le
\epsilon
+
\log|\mathcal A|
+
\epsilon,
\end{align*}
which proves the claim.
\end{proof}

This bound separates two effects. The approximation term measures alignment between the induced reward $r_q$ and the true task reward on the constrained MDP. The entropy term reflects the fact that the reverse-KL selector solves a maximum-entropy problem rather than an unregularized one. If the comparison is made against the entropy-regularized task objective, the entropy term cancels and the bound becomes $2\epsilon$.

No reference policy is universally optimal. This is not a weakness of KL specifically, but a consequence of under-specification.

\begin{proposition}[No Universal Selector]
\label{prop:no-universal-selector}
Suppose $\mathcal D_{\min}$ contains at least two distinct visitation distributions. For any selector $f$ that returns a single $d_f\in\mathcal D_{\min}$, there exists a linear reward $r$ for which $d_f$ is not optimal over $\mathcal D_{\min}$.
\end{proposition}

\begin{proof}
Choose any $d'\in\mathcal D_{\min}$ with $d'\neq d_f$. Since the two visitation distributions differ, there exists a state-action pair $(s,a)$ such that $d'(s,a)\neq d_f(s,a)$. Let $c = 1$ if $d'(s, a) > d_f(s, a)$ and -1 otherwise.
Define a reward function
\begin{equation}
r(s', a') = \begin{cases}
c & s = s' \land a = a' \\
0 & \textrm{otherwise}
\end{cases}
\end{equation}

Then
\begin{equation}
\langle d'-d_f,r\rangle = \left| d'(s, a) - d_f(s, a) \right| > 0.
\end{equation}
Thus $d_f$ is not optimal for reward $r$ over $\mathcal D_{\min}$.
\end{proof}

The dependence on the regularizer vanishes when the intervention signal is sufficiently informative.

\begin{proposition}[Informativeness Regime]
\label{prop:informativeness}
If
\begin{equation}
\sup_{d,d'\in\mathcal D_{\min}}
\left|
\langle d-d',r^\star\rangle
\right|
\le
\epsilon,
\end{equation}
then every selector returning an element of $\mathcal D_{\min}$ is $\epsilon$-optimal under $r^\star$.
\end{proposition}

\begin{proof}
Let $d^\star\in\argmax_{d\in\mathcal D_{\min}}\langle d,r^\star\rangle$ and let $\hat d\in\mathcal D_{\min}$ be any selected visitation distribution. By assumption, $\langle d^\star-\hat d,r^\star\rangle \le \epsilon$. Therefore $\hat d$ is $\epsilon$-optimal.
\end{proof}

Taken together, these results separate the roles of intervention and regularization. Intervention defines the feasible geometry through the constrained MDP. Regularization induces an objective on this reduced domain. The quality of the selected policy depends on whether that induced objective is aligned with the task structure that remains after interventions have ruled out inadmissible behavior.

\subsection{Lemma~\ref{lem:rql-equivalence}}
\label{apdx:proof-rql-equivalence}

\begin{proof}
We begin by rewriting the constrained RL problem.
\begin{align}
&\argmax_{\pi} - \E_{s_t, a_t} \left[ \sum_{t=0}^\infty \gamma^t \phi(s_t, a_t) \right] - \omega \E_{s_t} \left[ \sum_{t=0}^\infty \gamma^t D_\textrm{KL} \left[ \pi ( s_t) \parallel \pi_0( s_t) \right] \right] \\
&= \argmax_{\pi} - \E_{s_t, a_t} \left[ \sum_{t=0}^\infty \gamma^t \phi(s_t, a_t) \right] - \\
&\qquad \omega \E_{s_t} \left[ \sum_{t=0}^\infty \gamma^t \sum_{a \in \mathcal A} \pi(a \mid s_t) \left( \log \pi(a \mid s_t) - \log \pi_0(a \mid s_t) \right) \right] \nonumber \\
&= \argmax_{\pi} -\E_{s_t, a_t} \left[ \sum_{t=0}^\infty \gamma^t \phi(s_t, a_t) \right] - \omega \E_{s_t} \left[ \sum_{t=0}^\infty \gamma^t \sum_{a \in \mathcal A} \pi(a \mid s_t) \log \pi(a \mid s_t) \right] + \\
&\qquad \omega \E_{s_t, a_t} \left[ \sum_{t=0}^\infty \gamma^t \log \pi_0(a_t \mid s_t) \right] \nonumber \\
&= \argmax_{\pi} \E_{s_t, a_t} \left[ \sum_{t=0}^\infty \gamma^t \left( -\phi(s_t, a_t) + \omega \log \pi_0 (a_t \mid s_t) \right) \right] + \omega \E_{s_t} \left[ \sum_{t=0}^\infty \gamma^t \mathcal H \left[ \pi( s_t) \right] \right] \\
&= \argmax_{\pi} \E_{s_t, a_t} \left[ \sum_{t=0}^\infty \gamma^t \left( \underbrace{-\phi(s_t, a_t) + \omega \log \pi_0 (a_t \mid s_t)}_{\tilde r(s_t, a_t)} + \omega \mathcal H \left[ \pi(  s_t) \right] \right) \right] \\
&= \argmax_{\pi} J_\textrm{MaxEnt}(\pi \mid \tilde r)
\end{align}

Let $\pi = \textrm{RQL}(\pi_0, -\phi)$. We will show that $\pi = \argmax_{\pi'} J_\textrm{MaxEnt} (\pi' \mid \tilde r)$.
By the definition of Residual Q-Learning, we have that $\pi = \argmax_{\pi'} J_\textrm{MaxEnt}(\pi \mid r_0 - \phi)$ for any $r_0$ such that $\pi_0 = \argmax_{\pi'} J_\textrm{MaxEnt}(\pi' \mid r_0)$.
It therefore suffices to show that $\pi_0$ is optimal for the reward function $r_0 (s, a) := \alpha \log \pi_0 (a \mid s)$.
Expanding out the max-ent objective,
\begin{align}
&\argmax_{\pi'} J_\textrm{MaxEnt}(\pi' \mid r_0) \\
&= \argmax_{\pi'} \E_{s_t, a_t} \left[ \sum_{t=0}^\infty \gamma^t \left( \alpha \log \pi_0 (a_t \mid s_t) + \alpha \mathcal H \left[ \pi' ( s_t) \right] \right) \right] \\
&= \argmax_{\pi'} \E_{s_t, a_t} \left[ \sum_{t=0}^\infty \gamma^t \left( \log \pi_0 (a_t \mid s_t) - \log \pi' (a_t \mid s_t) \right) \right] \\
&= \argmin_{\pi'} \E_{s_t, a_t} \left[ \sum_{t=0}^\infty \gamma^t \left( \log \pi' (a_t \mid s_t) - \log \pi_0 (a_t \mid s_t) \right) \right] \\
&= \argmin_{\pi'} \E_{s_t} \left[ \sum_{t=0}^\infty \gamma^t D_\textrm{KL} \left[ \pi' ( s_t) \parallel \pi_0 ( s_t) \right] \right] \\
&= \argmin_{\pi'} \sum_{s \in \mathcal S} d^{\pi'} (s) D_\textrm{KL} \left[ \pi' ( s_t) \parallel \pi_0 ( s_t) \right] \\
&= \pi_0
\end{align}
\end{proof}

\subsection{Theorem \ref{thm:grad-kl-div}}
\label{apdx:proof-grad-kl-div}

We first compute the derivative of this with respect to $\hat Q$.
To describe $|\mathcal S|\cdot |\mathcal A| \times |\mathcal S| \cdot |\mathcal A|$ matrices with a row and column for each state-action pair, we will enumerate all such pairs $(s_i, a_i)$ for $i \in \{ 1, \ldots, |\mathcal S| \cdot |\mathcal A| \}$.

\begin{lemma} \label{psi-q-deriv}
Let $\hat Q$ be a Q-function corresponding to $\hat \pi$.
The partial derivative of $\Psi$ with respect to $\hat Q(s, a)$ is
\begin{equation}
\frac{\partial \, \Psi}{\partial \, \hat Q(s, a)} = \frac{d^{\pi_E}(s)}{\alpha } \left( \hat \pi(a \mid s) - \pi_E(a \mid s) \right)
\end{equation}
\end{lemma}

\begin{proof}
Expanding out the definition of the KL divergence,
\begin{equation}
\Psi = \sum_{s \in \mathcal S} d^{\pi_E}(s) \sum_{a \in \mathcal A} \pi_E(a \mid s) \left( \log \pi_E(a \mid s) - \log \hat \pi(a \mid s) \right) \nonumber
\end{equation}
Thus,
\begin{equation}
\frac{\partial \, \Psi}{\partial \, \hat \pi(a \mid s)} = \frac{- d^{\pi_E}(s) \pi_E(a \mid s)}{\hat \pi(a \mid s)}
\end{equation}

For the second partial derivative, note that $\hat \pi(\cdot \mid s)$ is the softmax of $\hat Q(s, \cdot)$:
\begin{equation}
\hat \pi (a \mid s) = \frac{\exp \left( \frac{1}{\alpha} \hat Q(s, a) \right)}{\sum_{a'} \exp \left( \frac{1}{\alpha} \hat Q(s, a') \right)}
\end{equation}
and so
\begin{equation}
\frac{\partial\, \hat \pi(a_i \mid s_i)}{\partial\,  \hat Q(s_j, a_j)} = \mathbf{1}[s_i = s_j] \cdot \frac{1}{\alpha} \hat \pi(a_i \mid s_i) \cdot \left( \mathbf{1}[a_i = a_j] - \hat \pi(a_j \mid s_j) \right)
\end{equation}

Combining these two derivatives with the chain rule,
\begin{align}
&\frac{\partial \, \Psi}{\partial \, \hat Q(s, a)} \\
&= \sum_{\substack{s' \in \mathcal S \\ a' \in \mathcal A}} \frac{\partial \, \Psi}{\partial \, \hat \pi(a' \mid s')} \frac{\partial \, \hat \pi (a' \mid s')}{\partial \, \hat Q(s, a)} \\
&= \sum_{\substack{s' \in \mathcal S \\ a' \in \mathcal A}} \frac{- d^{\pi_E}(s') \pi_E(a' \mid s')}{\hat \pi(a' \mid s')} \cdot \mathbf{1}[s = s'] \cdot \frac{1}{\alpha} \hat \pi(a' \mid s') \cdot \left( \mathbf{1}[a = a'] - \hat \pi(a \mid s) \right) \\
&= \frac{- d^{\pi_E}(s)}{\alpha} \sum_{a' \in \mathcal A}  \pi_E(a' \mid s) \cdot \left( \mathbf{1}[a = a'] - \hat \pi(a \mid s) \right) \\
&= \frac{d^{\pi_E} (s)}{\alpha} \left( \hat \pi (a \mid s) - \pi_E(a \mid s) \right)
\end{align}
\end{proof}

\begin{lemma}
Let $\hat r$ be a reward function with corresponding soft-Q function $\hat Q$ and policy $\hat \pi$.
Let $W_{\hat \pi}$ be the state-action transition matrix given by
\begin{equation}
W_{\hat \pi} [i, j] = \mathcal T(s_j \mid s_i, a_i) \hat \pi(a_j \mid s_j)
\end{equation}
Then the Jacobian $J$ with entries 
\begin{equation}
J[i, j] = \frac{\partial \, \hat Q(s_i, a_i)}{\partial \, \hat r(s_j, a_j)}
\end{equation}
is equal to
\begin{equation}
J = \left( \mathbb I - \gamma W_{\hat \pi} \right)^{-1}
\end{equation}
\end{lemma}

\begin{proof}
See \citet{rl_theory_book}, Chapter 1.
\end{proof}

Combining these two gradients, 
\begin{equation} \label{eqn:grad-psi-r-unsimplified}
\nabla_{\hat r} \Psi = \left( \nabla_{\hat Q} \Psi \right)^\top J 
\end{equation}

We now prove Theorem~\ref{thm:grad-kl-div}.

\begin{proof}
Observe that the state-action visitation vectors satisfy
\begin{equation}
\mathbf d^\pi_{t+1} = \mathbf d^\pi_t W_\pi = \mathbf d_0^\pi W_\pi^{t+1}
\end{equation}
for any policy $\pi$.

In the following equation we use $\times$ to denote matrix multiplication and the notation $\left[ f(s_j, a_j )\right]$ to denote a $|\mathcal S| \cdot |\mathcal A|$ dimensional row vector whose $j$\textsuperscript{th} entry is given by $f$.
We can simplify (\ref{eqn:grad-psi-r-unsimplified}) as follows.
\begin{align}
&\nabla_{\hat r} \Psi \\
&=\nabla_{\hat r} \sum_{s \in \mathcal S} d^{\pi_E} (s) D_\textrm{KL} \left[ \pi_E( s) \parallel \hat \pi( s) \right] \\
&= \left[ \frac{d^{\pi_E}(s_j)}{\alpha} \left( \hat \pi(a_j \mid s_j) - \pi_E(a_j \mid s_j) \right) \right] \times \left( \mathbb I - \gamma W_{\hat \pi} \right)^{-1} \nonumber \\
&= \frac{1 - \gamma}{\alpha} \cdot \left( \sum_{t=0}^\infty \left[ \gamma^t d_t^{\pi_E} (s_j) \left( \hat \pi(a_j \mid s_j) - \pi_E(a_j \mid s_j) \right) \right] \right) \times \left( \sum_{t'=0}^\infty \gamma^{t'} W_{\hat \pi}^{t'} \right) \\
&= \frac{1 - \gamma}{\alpha} \cdot \left( \sum_{t=0}^\infty\left[  \gamma^t d_t^{\pi_E} (s_j) \hat \pi(a_j \mid s_j) - \gamma^t d_t^{\pi_E} (s_j, a_j) \right] \right) \times \left( \sum_{t'=0}^\infty \gamma^{t'} W_{\hat \pi}^{t'} \right)
\end{align}
Changing the indexing over the product of terms from $t$ and $t'$ to $x=t + t'$ and $y=t$ yields
\begin{align}
&= \frac{1 - \gamma}{\alpha} \sum_{x=0}^\infty \sum_{y=0}^x \left[ \gamma^y d_y^{\pi_E} (s_j) \hat \pi(a_j \mid s_j) - \gamma^y d_y^{\pi_E} (s_j, a_j) \right] \times \gamma^{x-y} W_{\hat \pi}^{x - y} \\
&= \frac{1 - \gamma}{\alpha} \sum_{x=0}^\infty \gamma^x \sum_{y=0}^x \left[ d_y^{\pi_E} (s_j) \hat \pi(a_j \mid s_j) - d_y^{\pi_E} (s_j, a_j) \right] \times  W_{\hat \pi}^{x - y} \\
&= \frac{1 - \gamma}{\alpha} \sum_{x=0}^\infty \gamma^x \left( \left[ d_0^{\pi_E} (s_j) \hat \pi(a_j \mid s_j) - d_0^{\pi_E} (s_j, a_j) \right] W_{\hat \pi}^x  \sum_{y=1}^x \left[ d_y^{\pi_E} (s_j) \hat \pi(a_j \mid s_j) -  d_y^{\pi_E} (s_j, a_j) \right] W_{\hat \pi}^{x - y} \right) \\
&= \frac{1 - \gamma}{\alpha} \sum_{x=0}^\infty \gamma^x \left( \left[ d_0^{\hat \pi} (s_j, a_j) - d_0^{\pi_E} (s_j, a_j) \right] W_{\hat \pi}^x  \sum_{y=1}^x \left[ d_y^{\pi_E} (s_j) \hat \pi(a_j \mid s_j) -  d_y^{\pi_E} (s_j, a_j) \right] W_{\hat \pi}^{x - y} \right)
\end{align}
We can rewrite this using the visitation vectors as
\begin{align}
&= \frac{1 - \gamma}{\alpha} \sum_{x = 0}^\infty \gamma^x \left( \mathbf d_0^{\hat \pi} W_{\hat \pi}^x - \mathbf d_0^{\pi_E} W_{\hat \pi}^x + \sum_{y=1}^x \left( \mathbf d_{y-1}^{\pi_E} W_{\hat \pi} - \mathbf d_y^{\pi_E} \right)  W_{\hat \pi}^{x - y} \right) \\
&= \frac{1 - \gamma}{\alpha} \sum_{x = 0}^\infty \gamma^x \left( \mathbf d_0^{\hat \pi} W_{\hat \pi}^x - \mathbf d_0^{\pi_E} W_{\hat \pi}^x + \sum_{y=1}^x \left( \mathbf d_{y-1}^{\pi_E} W_{\hat \pi}^{x-y+1} - \mathbf d_y^{\pi_E} W_{\hat \pi}^{x - y} \right) \right) 
\end{align}
Recognizing this as a telescoping sum yields the final result.
\begin{align}
&= \frac{1 - \gamma}{\alpha} \sum_{x = 0}^\infty \gamma^x \left( \mathbf d_x^{\hat \pi} - \mathbf d_x^{\pi_E}  \right) \\
&= \frac{1}{\alpha} \left( \mathbf d^{\hat \pi} - \mathbf d^{\pi_E} \right)
\end{align}
\end{proof}

\section{Experiments}

\subsection{Experimental Setup}
\label{apdx:experimental-setup}

\paragraph{Environments}
We perform experiments using Gymnasium environments \citep{gym_towers_2024} and algorithms based on Stable Baselines 3 \citep{stable-baselines3}.

\paragraph{Compute Resources}
All experiments are run on a single Titan RTX or equivalent GPU, with each experiment taking no more than several hours to run. We repeat experiments (typically 5-10 times) with different random seeds to estimate confidence intervals.
We estimate the total amount of compute for this paper to be 1 month of GPU time.
Often, several experiment jobs can be run in parallel on a single GPU to improve efficiency. 

\paragraph{Input Policies}
The optimal policy $\pi^*$ is trained using standard RL on the true reward function.
From the optimal policy $\pi^*$, we create demonstration datasets and run behavior cloning \citep{pomerleau1988alvinn} to get prior policies $\pi_0$ of various levels of ability.

\paragraph{Algorithms}
We use Residual SAC \citep{sac_haarnoja_2018, rql_li_2023} as the implementation for RIFT due to its improved stability.
For a given environment, we use the SAC parameters recorded in RL Zoo \citep{rl-zoo3}, with no adjustments made for these intervention experiments.
The SAC actor network is initialized to $\pi_0$.
We parametrize our experiments by scaling the prior policy coefficient $\omega$.
Unless stated otherwise, we use $\omega = 0.001$ for RIFT.
Note that the magnitude of $\omega$ is affected by the magnitude of uncertainty in the prior policy, as described in Sec.~\ref{apdx:action_distribution}.
Throughout, we treat RLIF~\citep{rlif_luo_2024} as the unregularized baseline corresponding to $\omega = 0$.

\paragraph{Evaluation} For all experiments, we evaluate the policy without the expert intervention model (i.e. in the original environment) and measure mean episode reward and success rate (where the definition of success depends on the specific environment).
We also evaluate the policy with interventions to measure the intervention rate, defined as whether the expert intervenes at any point in the full rollout.
All evaluation metrics use deterministic actions (i.e. taking the mode of the action distribution), while training uses stochastic actions (i.e. sampling from the action distribution).
All quantitative metrics are reported with 95\% confidence intervals assuming normally distributed measurements.

\paragraph{Critic Warm-up}
We find it beneficial to warm-up the critic network by keeping the actor network frozen for a period at the start of training.
We train for 2,500,000 steps total and freeze the actor for the first 200,000.
This freeze period was determined by running RQL with zero residual reward and finding the parameters that kept the performance stable.

\begin{table}[]
    \centering
    \caption{Success rate under intervention delay.}
    \begin{tabular}{rll}
    \toprule
       Delay & RLIF        & RIFT        \\
    \midrule
           0 & 0.62 ± 0.16 & \textbf{0.99} ± 0.00 \\
           4 & 0.46 ± 0.21 & \textbf{0.99} ± 0.00 \\
           8 & 0.27 ± 0.25 & \textbf{1.00} ± 0.00 \\
          12 & 0.44 ± 0.19 & \textbf{0.99} ± 0.01 \\
          16 & 0.39 ± 0.20 & \textbf{0.99} ± 0.00 \\
          20 & 0.35 ± 0.27 & \textbf{0.99} ± 0.00 \\
          32 & 0.01 ± 0.02 & \textbf{0.99} ± 0.00 \\
          64 & 0.02 ± 0.03 & \textbf{0.99} ± 0.01 \\
    \bottomrule
    \end{tabular}
    \label{tab:delay}
\end{table}

\subsection{Environment Specific Details}
\label{apdx:environment-specifics}

\paragraph{Lunar Lander}
We use the SAC parameters from RL Zoo \citep{rl-zoo3} as shown in Table~\ref{table:lunar-lander}.
All other hyperparameters use the default values from the SAC implementation \citep{stable-baselines3}.
For RIFT we use $\omega = 0.001$.
We use the continuous action space version of the environment by setting \texttt{continuous=True}.
Success is defined as achieving an episode reward of 200 or better, as described in the environment documentation \citep{gym_towers_2024}.
For the intervention model we use $B=5,3,1$ for low, medium, and high informative interventions, respectively.

\begin{table}
\centering
\caption{Hyperparameters from RL Zoo for Lunar Lander and Bipedal Walker. All other hyperparameters use the SAC defaults.}
\begin{subtable}{0.48\textwidth}
  \caption{Hyperparameters for Lunar Lander}
  \label{table:lunar-lander}
  \begin{center}
        \begin{tabular}{lc}
          \toprule
          Parameter & Value \\
          \midrule
          Batch Size & 256 \\
          Buffer Size & 1,000,000 \\
          Entropy Coefficient & auto \\
          Gamma & 0.99 \\
          Gradient Steps & 1 \\
          Learning Rate & 7.3e-4 \\
          Learning Starts & 10,000 \\
          Policy Network Architecture & [400, 300] \\
          Tau & 0.01 \\
          Train Frequency & 1 \\
          \bottomrule
        \end{tabular}
  \end{center}
\end{subtable}
\hfill
\begin{subtable}{0.48\textwidth}
  \caption{Hyperparameters for Bipedal Walker}
  \label{table:bipedal-walker}
  \begin{center}
        \begin{tabular}{lc}
          \toprule
          Parameter & Value \\
          \midrule
          Batch Size & 256 \\
          Buffer Size & 300,000 \\
          Entropy Coefficient & auto \\
          Gamma & 0.98 \\
          Gradient Steps & 64 \\
          Learning Rate & 7.3e-4 \\
          Learning Starts & 10,000 \\
          Policy Network Architecture & [400, 300] \\
          Tau & 0.02 \\
          Train Frequency & 64 \\
          \bottomrule
        \end{tabular}
  \end{center}
\end{subtable}
\end{table}

\paragraph{Half Cheetah}
We follow RL Zoo \citep{rl-zoo3} and use the default SAC parameters with the exception of setting \texttt{learning\_starts} to 10,000.
We freeze the actor for the first 400,000 steps.
For RIFT we use $\omega = 0.01$.

\paragraph{Bipedal Walker}
We use the SAC parameters from RL Zoo \citep{rl-zoo3} as shown in Table~\ref{table:bipedal-walker}.
All other hyperparameters use the default values from the SAC implementation \citep{stable-baselines3}.
For RIFT we use $\omega = 0.01$.
We freeze the actor for the first 400,000 steps.
Success is defined as achieving an episode reward of 300 or better, as described in the environment documentation \citep{gym_towers_2024}.

\begin{figure}
    \centering
    \includegraphics[width=\textwidth]{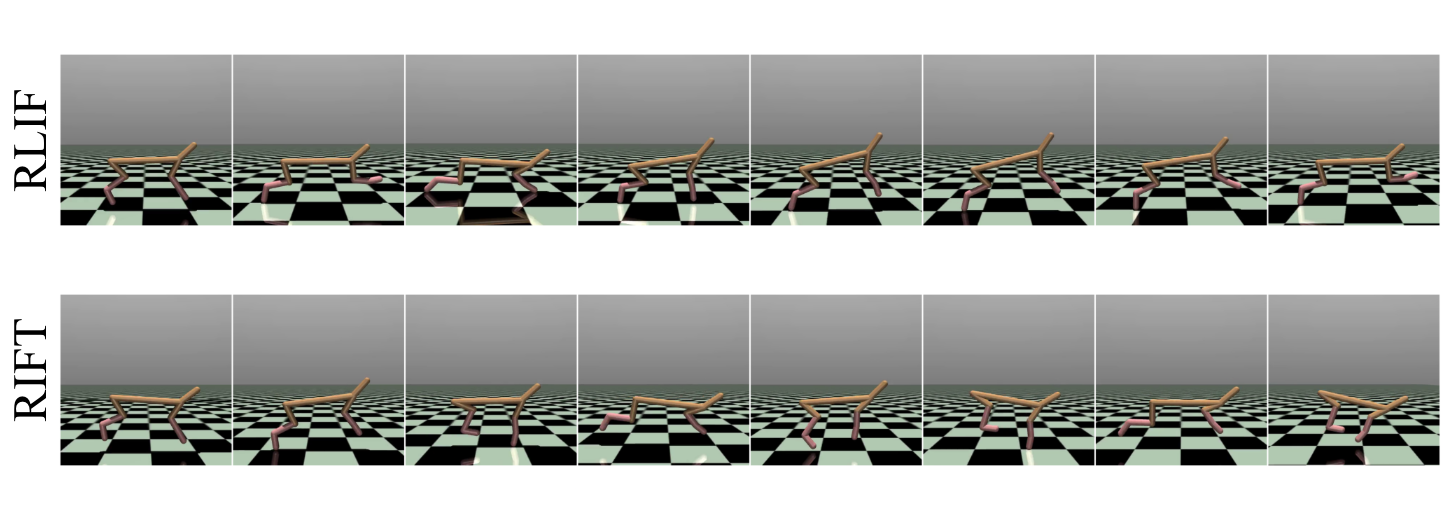}
    \caption{Video frames from Half Cheetah of RLIF and RIFT. RLIF bounces up and down while moving forward slowly. RIFT is more efficient and travels faster.}
    \label{fig:half-cheetah-video}
\end{figure}

\subsection{Noisy Interventions}
\label{apdx:noisy-interventions}

\begin{table}[t]
    \centering
    \caption{Success rate under stochastic interventions with false positive and false negative noise.}
    \label{table:noisy-interventions}
    \begin{subtable}{0.48\textwidth}
      \centering
      \caption{False positive intervention noise.}
      \begin{tabular}{rll}
\toprule
   \makecell{False Positive\\Probability\\Per Rollout} & RLIF        & RIFT        \\
\midrule
0 & 0.66 ± 0.18 & \textbf{0.99} ± 0.00 \\
0.2 & 0.63 ± 0.23 & \textbf{0.99} ± 0.01 \\
0.4 & 0.58 ± 0.21 & \textbf{0.97} ± 0.01 \\
0.6 & 0.56 ± 0.14 & \textbf{0.95} ± 0.02 \\
0.8 & 0.47 ± 0.15 & \textbf{0.92} ± 0.02 \\
0.9 & 0.67 ± 0.14 & \textbf{0.91} ± 0.03 \\
0.95 & 0.38 ± 0.20 & \textbf{0.85} ± 0.02 \\
0.98 & 0.43 ± 0.19 & \textbf{0.78} ± 0.07 \\
0.99 & 0.59 ± 0.24 & 0.61 ± 0.17         \\
\bottomrule
\end{tabular}
    \end{subtable}
    \begin{subtable}{0.48\textwidth}
      \centering
      \caption{False negative intervention noise.}
      \begin{tabular}{rll}
\toprule
   \makecell{False Negative\\Probability\\Per Step} & RLIF        & RIFT        \\
\midrule
0 & 0.66 ± 0.18 & \textbf{0.99} ± 0.00 \\
0.2 & 0.63 ± 0.25 & \textbf{0.99} ± 0.00 \\
0.4 & 0.58 ± 0.18 & \textbf{0.99} ± 0.01 \\
0.6 & 0.56 ± 0.22 & \textbf{0.99} ± 0.01 \\
0.8 & 0.31 ± 0.18 & \textbf{0.99} ± 0.00 \\
0.9 & 0.45 ± 0.24 & \textbf{0.99} ± 0.01 \\
0.95 & 0.53 ± 0.17 & \textbf{0.97} ± 0.02 \\
0.98 & 0.34 ± 0.22 & \textbf{0.98} ± 0.01 \\
0.99 & 0.20 ± 0.19 & \textbf{0.95} ± 0.02 \\
\bottomrule
\end{tabular}
    \end{subtable}
\end{table}

We evaluate stochastic interventions that violate the assumption that an intervention-free policy exists.
There are two types of intervention noise: false positive interventions where an intervention occurs despite the intervention criteria not being met, and false negative interventions where no intervention occurs despite the intervention criteria being met.
In emergency-stop interventions, these two types of noise have vastly different impacts.
False negative interventions weaken the strength of the intervention signal but do not fundamentally change any other properties of the algorithm; in expectation it is as if we replace the -1 reward from interventions with $-(1-p_\textrm{fn})$, where $p_\textrm{fn}$ is the probability of a false negative on each step where the intervention criteria are met.

In contrast, false positive interventions end the episode, preventing the policy from observing later states.
Thus, false positive interventions impact not only the immediate state-action pair where the intervention takes place, but all subsequent states that are no longer observed from the episode.
We therefore measure the probability of having a false positive intervention over a full episode by taking the approximate average episode length for successful trajectories $T$ and computing $1 - (1-p_\textrm{fp})^T$, where $p_\textrm{fp}$ is the probability of a false positive intervention on each step where the intervention criteria are not met.

Results under both types of intervention noise are shown in Table~\ref{table:noisy-interventions}.
As predicted above, RIFT shows remarkable robustness to false negative noise.
Performance is somewhat sensitive to extremely high levels of false positive noise (e.g. when even optimal trajectories get interventions over 90\% of the time) due to the data-limiting effect described above, yet RIFT still provides statistically-significant improvement over RLIF.

\subsection{Interpretable Intervention Strategies}
\label{apdx:heuristic-interventions}

The intervention strategy in Eqn.~\ref{eqn:q-diff-intervention-strategy} is extremely useful but has two limitations.
First, it is a global intervention strategy; all states receive equal supervision.
In practical settings, a human superviser might be very strict in safety-critical states, but be much less strict otherwise.
This locality is another factor that impacts the informativeness of the intervention strategy, and has been underexplored in prior works.
Second, using a learned model as the basis of the itnervention strategy makes it difficult to describe the specific behaviors that trigger interventions.

\begin{figure}[t]
    \centering
    \begin{subfigure}{0.32\textwidth}
      \centering
      \includegraphics[width=\textwidth]{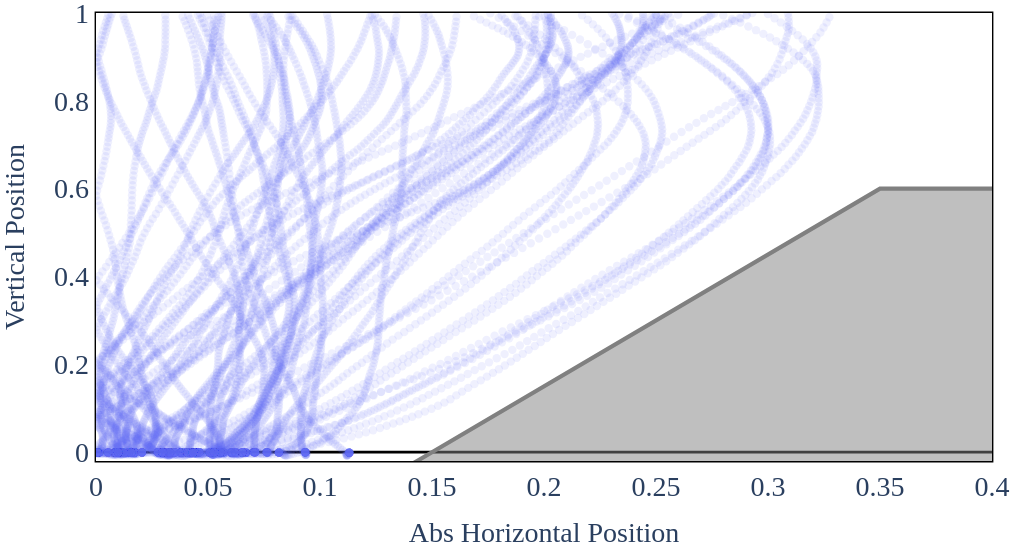}
    \end{subfigure}
    \begin{subfigure}{0.32\textwidth}
      \centering
      \includegraphics[width=\textwidth]{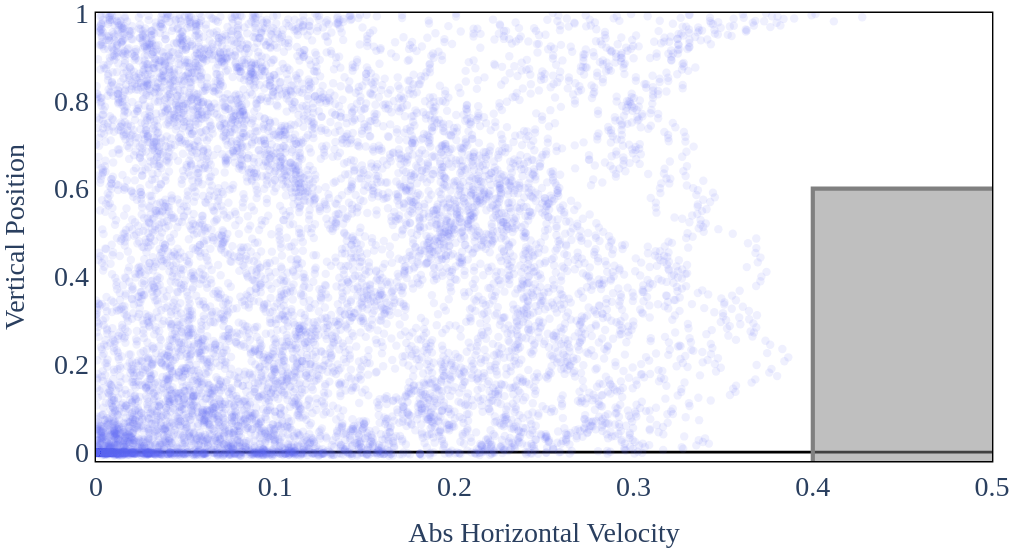}
    \end{subfigure}
    \begin{subfigure}{0.32\textwidth}
      \centering
      \includegraphics[width=\textwidth]{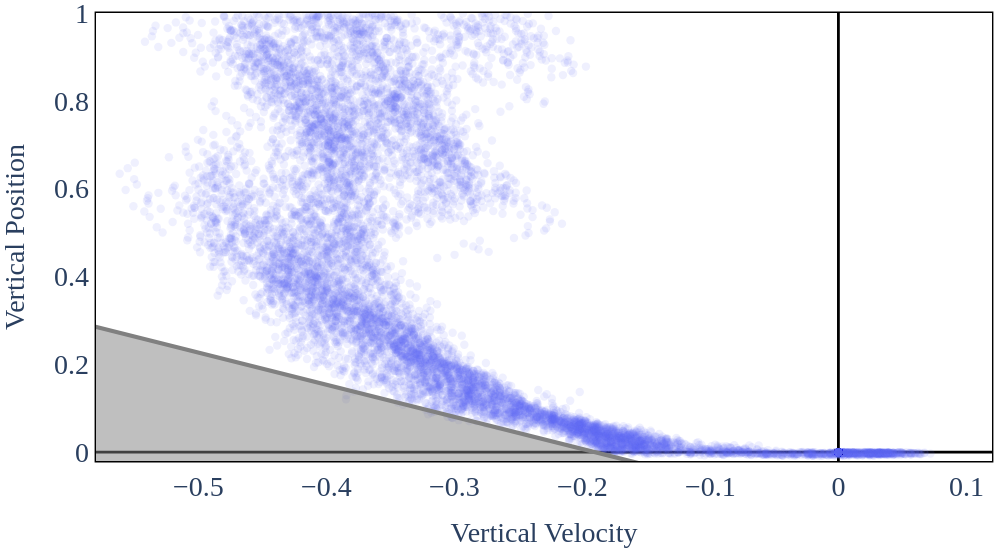}
    \end{subfigure}
    \caption{Three heuristic intervention strategies for lunar lander. Traces from an expert trajectory are plotted in blue. From this, we define heuristic intervention strategies to intervene if the state enters the gray region.}
    \label{fig:heuristic-strategies}
\end{figure}

\begin{figure}
    \centering
    \begin{subfigure}{0.45\textwidth}
      \centering
      \includegraphics[width=\textwidth]{figs/heuristic_intervention_plot_cropped.png}
      \caption{Quantitative results.}
    \end{subfigure}
    \hfill
    \begin{subfigure}{0.5\textwidth}
      \centering
      \includegraphics[width=\textwidth]{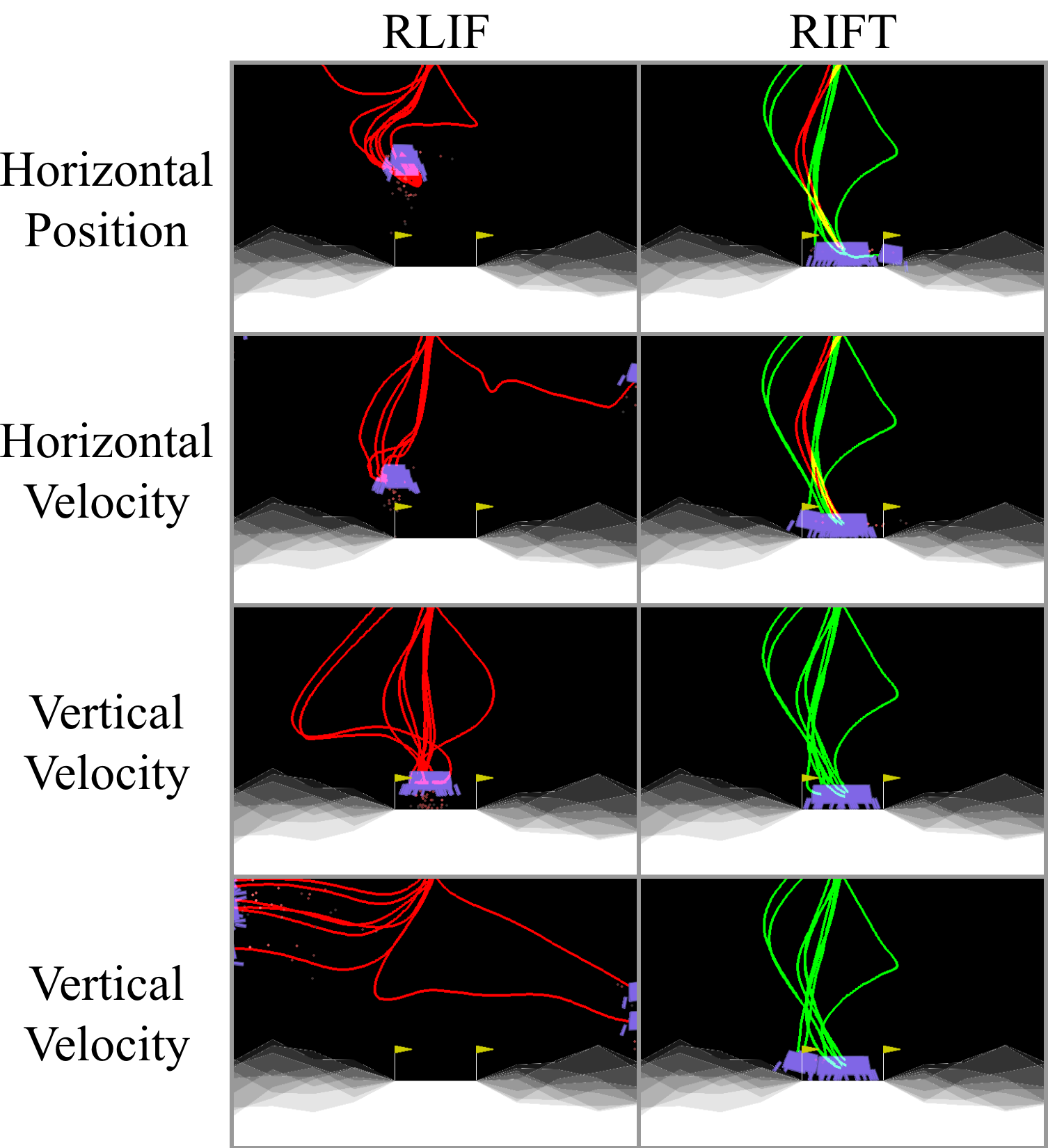}
      \caption{Example policy rollouts.}
    \end{subfigure}
    \caption{Experimental results under heuristic intervention strategies. While both RLIF and RIFT find ways to avoid interventions, only RIFT successfully completes the task. Looking at example rollouts, the RLIF policies avoid getting close to the ground while RIFT reaches the landing area. The RIFT failure cases under the horizontal-based intervention strategies are due to crashing into the ground (something not captured by the horizontal-based intervention strategies). We depict two different policies trained with the vertical velocity intervention strategy to highlight the variability of RLIF between runs due to the under-specified nature of its objective function.}
    \label{fig:heuristic-strategies-results}
\end{figure}

In this section, we experiment with heuristic intervention strategies.
We do this for the Lunar Lander environment, where the low dimensional state space makes defining these heuristic intervention strategies much more straightforward.
This allows us to exactly describe the types of behaviors that trigger interventions, and control which areas of the state space receive interventions.
We start by gathering traces from an optimal policy and plotting various dimensions of the state space.
From this, we identify regions of the state space that the expert policy does not visit to define heuristic intervention strategies.
This analysis is shown in Fig.~\ref{fig:heuristic-strategies}.
We define three criteria for interventions, and can also combine them (i.e. intervene if any criteria is met).
Results are shown in Fig.~\ref{fig:heuristic-strategies-results}.
While both algorithms manage to avoid interventions, RLIF does so by avoiding the ground whereas RIFT learns to successfully complete the task.
RLIF also exhibits high variability across experiments due to the under-specified nature of its objective function.

We also experiment with choosing the intervention strategy to deliberately overlap with the expert visitation distribution.
As shown in Fig.~\ref{fig:heuristic-strategies-overlap}, we define three intervention strategies based on the vertical velocity to have no, partial, and full overlap with the expert state distribution, respectively.
Results are shown in Fig.~\ref{fig:heuristic-overlap-results}.
Once again, both algorithms manage to avoid interventions, but only RIFT does so by successfully completing the task.
Surprisingly, the RIFT policies exhibit new behaviors not observed in the expert or prior policies: approaching the landing area very slowly and circling back to attempt a second landing if the first one fails.
These examples show that even if $\pi_E \notin \mathcal D_\textrm{min}$, optimizing the RIFT objective can still yield policies that solve the task.

\begin{figure}
    \centering
    \begin{subfigure}{0.32\textwidth}
      \centering
      \includegraphics[width=\textwidth]{figs/lunar_lander_vertical_velocity_min_cropped.png}
    \end{subfigure}
    \begin{subfigure}{0.32\textwidth}
      \centering
      \includegraphics[width=\textwidth]{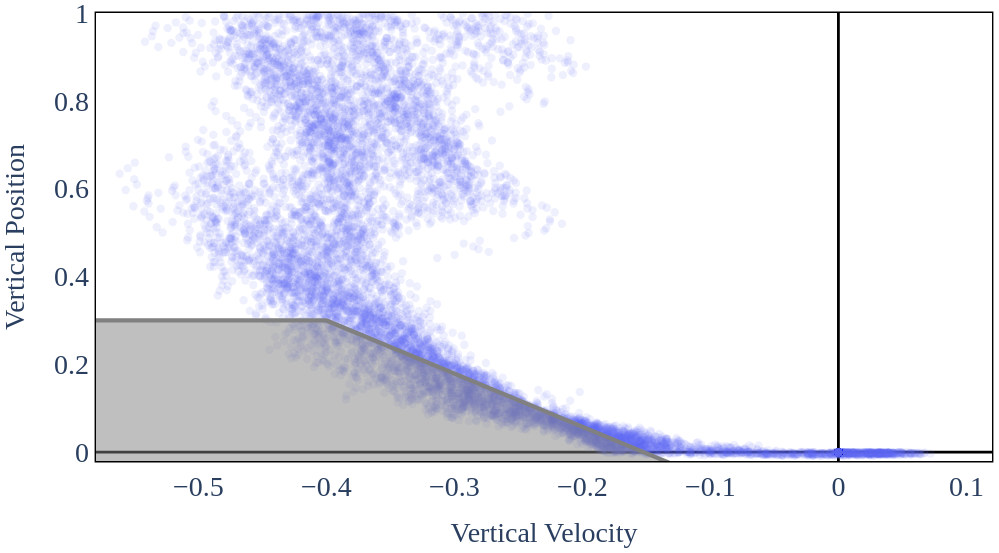}
    \end{subfigure}
    \begin{subfigure}{0.32\textwidth}
      \centering
      \includegraphics[width=\textwidth]{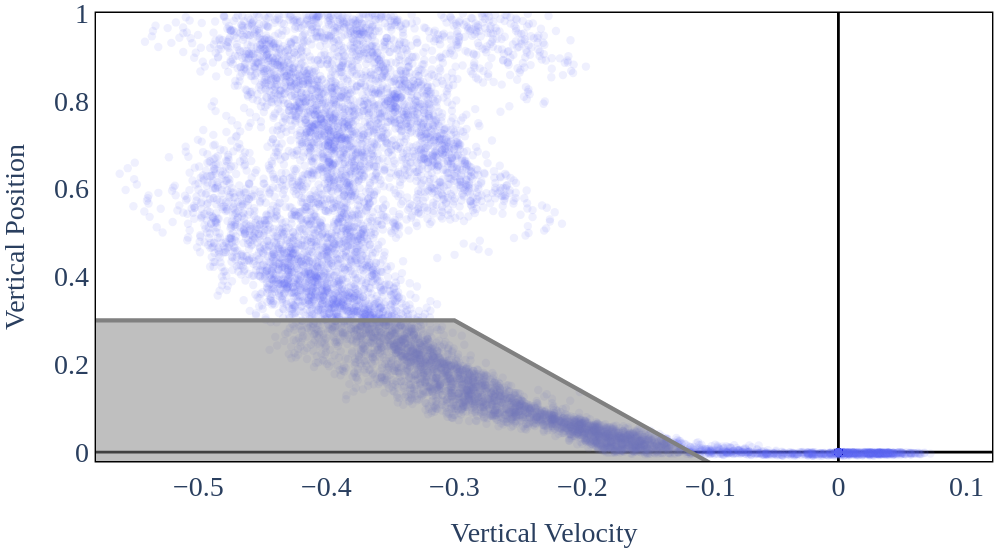}
    \end{subfigure}
    \caption{Three versions of the vertical velocity intervention strategy with no, partial, and full overlap with the expert policy state distribution. These intervention strategies break the assumption that the expert does not get intervened on.}
    \label{fig:heuristic-strategies-overlap}
\end{figure}

\begin{figure}
    \centering
    \begin{subfigure}{0.5\textwidth}
      \centering
      \includegraphics[width=\textwidth]{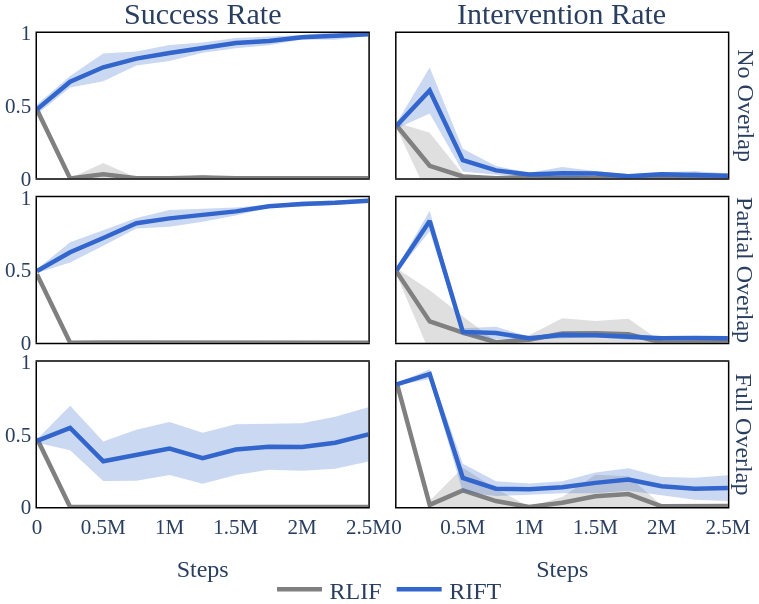}
      \caption{Quantitative results.}
    \end{subfigure}
    \begin{subfigure}{0.48\textwidth}
      \centering
      \includegraphics[width=\textwidth]{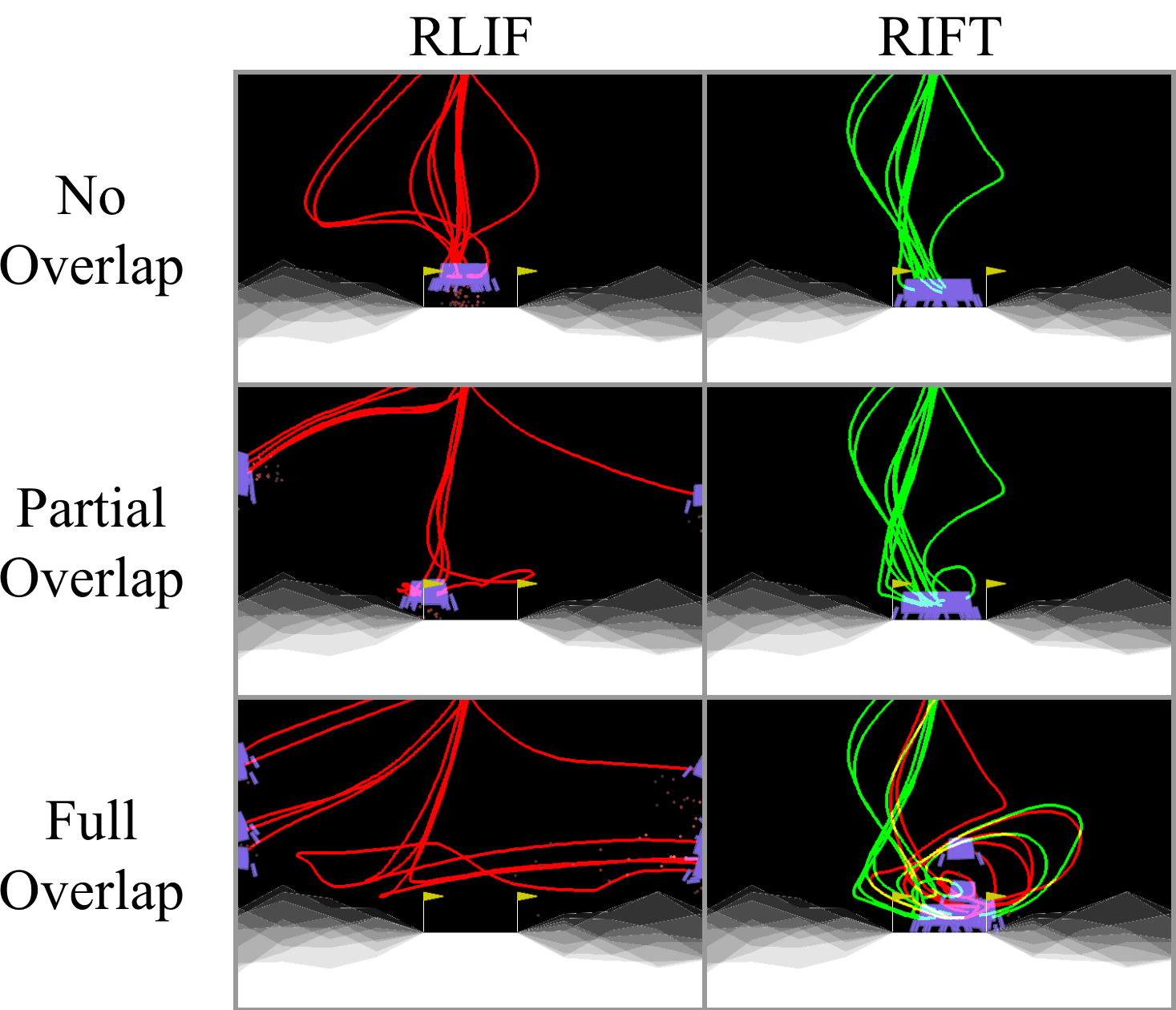}
      \caption{Example policy rollouts.}
    \end{subfigure}
    \caption{Experimental results under heuristic intervention strategies that deliberately overlap with the expert state-action visitation. Both algorithms learn to avoid interventions, but RLIF does so by avoiding the ground while RIFT learns to land successfully. Fascinatingly, when the intervention strategy overlaps with expert behavior, RIFT adapts by approaching the landing area more slowly than seen in expert demonstrations or the prior policy. The policies even demonstrate error recovery by circling around to attempt to land a second time if the first time does not succeed.}
  \label{fig:heuristic-overlap-results}
\end{figure}

\subsection{Termination vs. Truncation}
\label{apdx:termination_truncation}

In Algorithm \ref{alg:supervised-rollout}, when the expert triggers an e-stop the episode terminates.
It would seem natural to treat this as a termination when performing RL.
However, there are both theoretical and practical issues with doing so.

Including these early terminations changes the dynamics function $\mathcal T$ of the MDP.
Residual Q-Learning and the proofs in this work assume that the dynamics function remains constant.
From a theory perspective, the correct choice is to not mark intervention transitions as terminal.

Practically, marking intervention transitions as terminal makes the fine-tuning objective challenging.
The penalty of early termination dominates any other reward terms (the policy can't collect those rewards if the episode ends).
This nullifies the balance between the interventions and the prior policy.
Empirically, we observe that when using early termination the intervention rate goes down, even if the residual reward is set to always be zero.
In Figure~\ref{fig:termination} we run residual fine-tuning with termination upon intervention and zero residual reward (i.e. we do not set any reward due to interventions). Due to the early termination, the policy still learns to avoid interventions, ignoring the prior policy regularization and collapsing the success rate.
Early termination by itself is enough to force the policy to avoid interventions at all costs.

Instead of termination, we represent the intervention as a truncation: the episode ends, but the next state $s'$ is still used in the critic update.
For transition $(s, a, s', e)$, this changes the critic target from $- e$ to $- e + \gamma V(s')$.
This distinction is particularly important in RIL, where the intervention signal should shape behavior without redefining the task dynamics.

\begin{figure}[t]
    \vskip 0.2in
    \begin{center}
    \centerline{\includegraphics[width=0.5\linewidth]{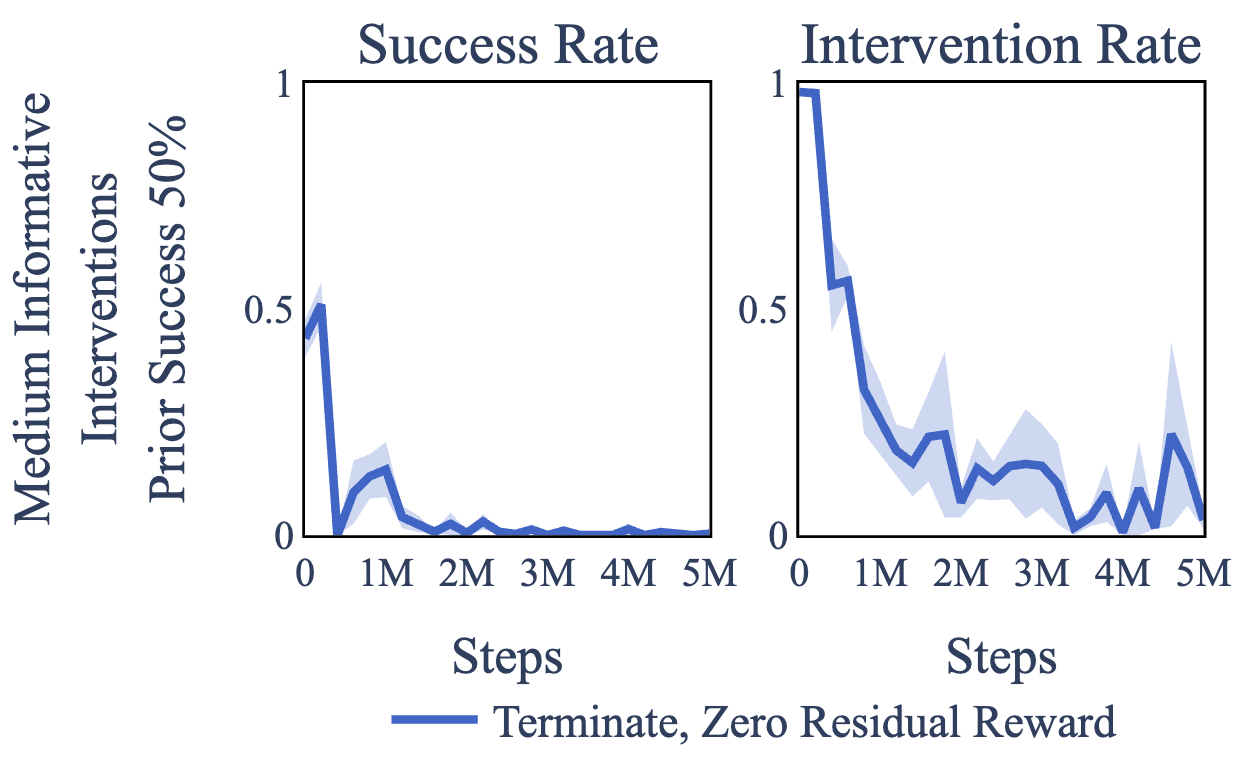}}
    \caption{Training runs using residual fine-tuning with termination instead of truncation on Lunar Lander where the residual reward is always 0. With no residual reward, the policy should converge to the prior policy. However, due to the added early terminations the policy instead shifts to avoid interventions, collapsing the success rate.}
    \label{fig:termination}
    \end{center}
\end{figure}

\subsection{Prior Policy Distribution}
\label{apdx:action_distribution}

The standard SAC policy in Stable Baselines 3 \citep{stable-baselines3} uses a squashed diagonal gaussian action distribution.
This distribution takes a $D$-dimensional diagonal gaussian distribution and passes it through a tanh non-linearity to squash all values into $[-1, 1]$.
The ``mode'' of this distribution (i.e. the action used for deterministic sampling) is defined by passing the mean of the gaussian through the tanh non-linearity.
Due to the squashing effect near -1 and 1, this is not necessarily the action with the highest likelihood, as shown in Figure~\ref{fig:action-dists}.
As $|\mu|$ and $\log \sigma$ increase, the squashing effect causes higher probability mass near the ends of the distribution.
Because of this, the ``mode'' is often not the maximum likelihood point of the distribution.

\begin{figure}[t]
    \vskip 0.2in
    \begin{center}
    \centerline{\includegraphics[width=0.5\linewidth]{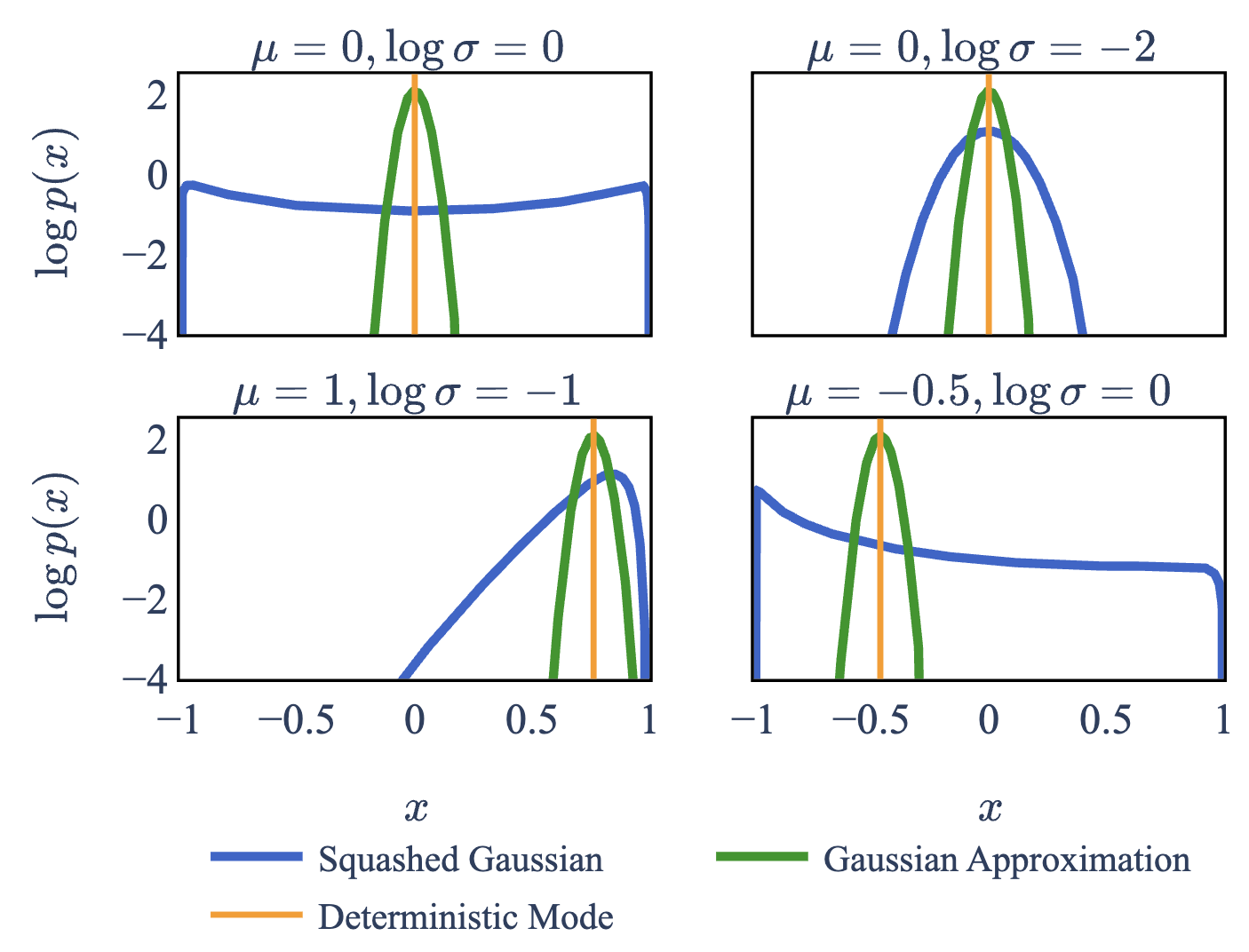}}
    \caption{Comparison of action distributions. Blue: the squashed gaussian distribution log probabilities with the specified mean and standard deviation. Orange: the ``mode'' of this distribution as defined in Stable Baselines 3, obtained by passing the mean of the distribution through the tanh non-linearity. Green: A standard gaussian distribution about the mode with standard deviation 0.05.}
    \label{fig:action-dists}
    \end{center}
\end{figure}

Empirically, we find that doing residual Q-learning \citep{rql_li_2023} with the original $\log \pi_0(a \mid s)$ from the squashed gaussian distribution can be unstable.
Let $\pi_0^\textrm{mode} (s)$ denote the deterministic action obtained by taking the ``mode'' of the squashed diagonal gaussian distribution with parameters determined by calling $\pi_0$ at $s$.
Even when doing residual Q-learning with zero residual reward and the actor initialized to the prior policy, the policy performance can collapse as shown in Figure~\ref{fig:approx-prior-zero-resid}.
We hypothesize this is because the residual Q-learning policy is shifting to maximize $\log \pi_0(a \mid s)$ with actions closer to -1 or 1 compared to $\pi_0^\textrm{mode} (s)$.
To address this, we approximate the prior with a gaussian $\pi_0(a \mid s) \approx \mathcal N(a \mid \pi_0^\textrm{mode}(s), \sigma^2  I)$.
With this approximation we observe much more stable performance: the residual Q-learning policy stays close to the prior policy when there is zero residual reward, and we observe smoother learning curves even with non-zero residual reward.
With this approximation
\begin{equation}
\omega \log \pi_0(a \mid s) \approx - \frac{\omega}{2 \sigma^2} \left \lVert a - \pi_0^\textrm{mode} (s)\right \rVert^2 + C
\end{equation}
and so the choice of $\sigma$ can be folded into the parameter $\omega$. All experiments in this work use this approximation with $\sigma = 0.05$, as shown in Figure \ref{fig:action-dists}.
Using $\omega = 0.001$ yields an overall coefficient of $\frac{\omega}{2 \sigma^2} = 0.2$.
We leave the question of whether alternative action distributions can remove the need for this approximation to future work.

\begin{figure}
    \vskip 0.2in
    \begin{center}
    \centerline{\includegraphics[width=0.6\linewidth]{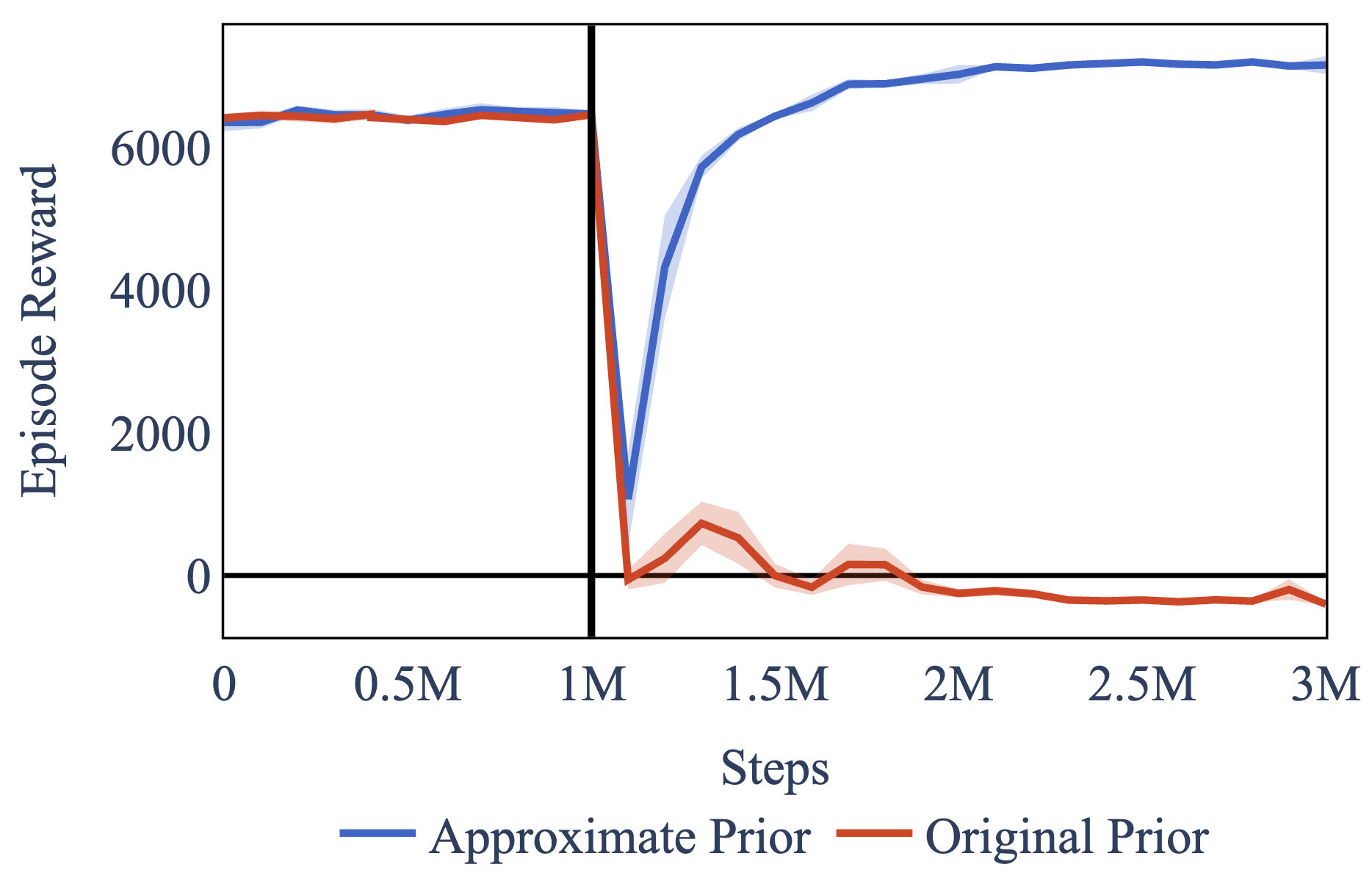}}
    \caption{Residual Q-Learning trained with zero residual reward on the Half Cheetah environment. The actor is initialized to the prior policy and frozen for the first 1 million steps to ensure the critic is  completely warmed up. Even with this warm up, the Residual Q-Learning policy performance collapses once the actor is unfrozen if using the original prior log probabilities based on the squashed diagonal gaussian. Using an approximate prior distribution centered on the prior policy mode keeps the policy close to the prior policy performance.}
    \label{fig:approx-prior-zero-resid}
    \end{center}
\end{figure}




\end{document}